\pdfoutput=1
\documentclass[twoside]{article}

\usepackage[accepted]{aistats2024}
\usepackage[round]{natbib}

\usepackage{lipsum,booktabs}
\usepackage{amsmath,mathrsfs,amssymb,amsfonts,bm,enumitem}
\usepackage{rotating}
\usepackage{pdflscape}
\usepackage{tcolorbox}
\usepackage{url,dsfont,nicefrac}
\allowdisplaybreaks
\usepackage{appendix}
\usepackage{multirow,makecell,tabularx}
\usepackage{algorithmic,algorithm}

\renewcommand{\tilde}{\widetilde}
\renewcommand{\hat}{\widehat}

\usepackage{hyperref}

\hypersetup{
    colorlinks=true,
    linkcolor=blue,
    citecolor=blue,
    filecolor=magenta,
    urlcolor=cyan
}

\def \A {\mathcal{A}}
\def \B {\mathbb{B}}
\def \B {\mathcal{B}}
\def \C {\mathcal{C}}

\def \E {\mathbb{E}}
\def \F {\mathcal{F}}

\def \M {\mathcal{M}}

\def \O {\mathcal{O}}
\def \P {\mathcal{P}}

\def \R {\mathbb{R}}
\def \S {\mathcal{S}}

\def \Ph {\hat{P}}
\def \Dp {D_{\psi}}

\def \x {\mathbf{x}}

\def \qh {\hat{q}}

\def \qt {\tilde{q}}

\def \Ot {\tilde{\O}}

\def \ellh {\hat{\ell}}
\def \thetah {\hat{\theta}}

\def \endenv {\hfill\raisebox{1pt}{$\triangleleft$}}

\def \regret {\textsc{Regret}}

\def \error {\textnormal{\textsc{Error}}}

\def \biasone {\textnormal{\textsc{Bias-\uppercase\expandafter{\romannumeral1}}}}
\def \biastwo {\textnormal{\textsc{Bias-\uppercase\expandafter{\romannumeral2}}}}

\def \termone {\textsc{Term-\uppercase\expandafter{\romannumeral1}}}
\def \termtwo {\textsc{Term-\uppercase\expandafter{\romannumeral2}}}

\def \sumk {\sum_{k=1}^K}
\def \sumh {\sum_{h=1}^H}

\usepackage{mathtools}
\let\norm\undefined 
\DeclarePairedDelimiter\norm{\lVert}{\rVert}

\DeclarePairedDelimiter\Bignorm{\Big\lVert}{\Big\rVert}
\DeclarePairedDelimiter\abs{\lvert}{\rvert}

\DeclarePairedDelimiter\Bigabs{\Big\lvert}{\Big\rvert}
\newcommand\inner[2]{\langle #1, #2 \rangle}

\DeclareMathOperator{\indicator}{\mathbb{I}}
\DeclareMathOperator*{\Reg}{Reg}

\DeclareMathOperator*{\argmax}{arg\,max}
\DeclareMathOperator*{\argmin}{arg\,min}

\usepackage{amsthm}
\usepackage{thmtools, thm-restate}

\newtheorem{myThm}{Theorem}

\newtheorem{myLemma}{Lemma}

\theoremstyle{definition}

\newtheorem{myDef}{Definition}

\newtheorem{myRemark}{Remark}
\newtheorem*{myProofSketch}{Proof Sketch}

\usepackage{graphicx,subfigure,color} 

\definecolor{wine_red}{RGB}{228,48,64}
\definecolor{DSgray}{cmyk}{0,1,0,0}

\newcommand\given[1][]{\:#1\vert\:}
\newcommand\givenn[1][]{\:#1\Big\vert\:}

\begin{document}

\runningtitle{Improved Algorithm for Adversarial Linear Mixture MDPs with Bandit Feedback and Unknown Trans.}

\twocolumn[

  \aistatstitle{Improved Algorithm for Adversarial Linear Mixture MDPs \\ with Bandit Feedback and Unknown Transition}

  \aistatsauthor{Long-Fei Li, Peng Zhao, Zhi-Hua Zhou}

  \aistatsaddress{National Key Laboratory for Novel Software Technology, Nanjing University, China\\
    School of Artificial Intelligence, Nanjing University, China\\
    \{lilf, zhaop, zhouzh\}@lamda.nju.edu.cn} ]

\begin{abstract}
  We study reinforcement learning with linear function approximation, unknown transition, and adversarial losses in the bandit feedback setting. Specifically, we focus on linear mixture MDPs whose transition kernel is a linear mixture model. We propose a new algorithm that attains an $\Ot(d\sqrt{HS^3K} + \sqrt{HSAK})$ regret with high probability, where $d$ is the dimension of feature mappings, $S$ is the size of state space, $A$ is the size of action space, $H$ is the episode length and $K$ is the number of episodes. Our result strictly improves the previous best-known $\Ot(dS^2 \sqrt{K} + \sqrt{HSAK})$ result in~\citet{ICLR'23:bandit-unknown-mixture} since $H \leq S$ holds by the layered MDP structure. Our advancements are primarily attributed to (\romannumeral1) a new least square estimator for the transition parameter that leverages the visit information of all states, as opposed to only one state in prior work, and (\romannumeral2) a new self-normalized concentration tailored specifically to handle non-independent noises, originally proposed in the dynamic assortment area and firstly applied in reinforcement learning to handle correlations between different states.
\end{abstract}

\section{INTRODUCTION}
\label{sec:intro}

\begin{table*}[t]
  \vspace{-3mm}
  \caption{Comparisons of regret bounds for adversarial tabular MDPs and linear mixture MDPs with bandit feedback and unknown transition in the literature. $S$ is the size of state space, $A$ is the size of action space, $K$ is the number of episodes and $H$ is the length of each episode, $d$ is the dimension of feature mapping.} \vspace{2mm}
  \label{tab:comparison}
  \centering
  \renewcommand*{\arraystretch}{1.2}
  \newcolumntype{Y}{>{\centering\arraybackslash}X}
  \begin{tabularx}{\textwidth}{cYYY}
    \toprule
    \multicolumn{1}{c}{}         & \textbf{Reference}                     & \textbf{Model}      & \textbf{Regret}                    \\ \midrule
    \multirow{3}{*}{Upper bound} & \citet{ICML'20:bandit-unknown-chijin}  & Tabular MDPs        & $\Ot(HS \sqrt{AK})$                \\
                                 & \citet{ICLR'23:bandit-unknown-mixture} & Linear Mixture MDPs & $\Ot(dS^2\sqrt{K} + \sqrt{HSAK})$  \\
                                 & This work                              & Linear Mixture MDPs & $\Ot(d\sqrt{HS^3K} + \sqrt{HSAK})$ \\ \midrule
    \multirow{2}{*}{Lower bound} & \citet{NIPS'18:Q-learning}             & Tabular MDPs        & $\Omega(H \sqrt{SAK})$             \\
                                 & \citet{ICLR'23:bandit-unknown-mixture} & Linear Mixture MDPs & $\Omega(dH\sqrt{K} + \sqrt{HSAK})$ \\ \bottomrule
  \end{tabularx}
\end{table*}

Reinforcement Learning (RL) studies the problem where a learner interacts with the environment sequentially and aims to improve the strategy over time. RL has achieved great success in the fields of games~\citep{arXiv'13:Atari}, robotic control~\citep{arXiv'17:PPO}, large language models~\citep{arXiv'23:GPT-4} and so on.

One of the most popular models to describe the RL problem is the Markov Decision Process (MDP)~\citep{Book:Puterman94}. Significant advances have emerged in learning MDPs with fixed or stochastic loss functions~\citep{JMLR'10:Jaksch-UCRL, ICML'17:Azar-minimax}, however, in many real-world applications, the losses may not be fixed or sampled from certain underlying distributions. As such, the pioneering works of~\citet{MatOR'09:online-MDP} and~\citet{MatOR'09:MDP-Yu} make the first step to formulate and study \emph{adversarial} MDPs, where the loss functions can be chosen adversarially and may change arbitrarily between each time step. Subsequently, many works explore different settings depending on the knowledge of the transition and the type of feedback received, whether it is full-information or bandit feedback~\citep{NIPS'13:O-REPS, ICML'19:unknown-transition-Rosenberg, NIPS'19:bandit-unknown-Rosenberg, ICML'20:bandit-unknown-chijin}. More detailed discussions are presented in Section~\ref{sec:related-work}.

Most existing works studying adversarial MDPs focus on the tabular setting, where the state and action space are small. Yet, in many problems, the state and action space can be large or even infinite. To overcome this challenge, a widely used approach in the literature is \emph{function approximation}, which reparameterizes the action-value function as a function over some feature mapping that maps the state and action to a low-dimensional space. In particular, linear function approximation has gained significant attention~\citep{COLT'20:Jin-linear-mdp, ICML'20:Ayoub-mixture, COLT'21:Zhou-mixture-minimax, NeurIPS'23:linearMDP}. Amongst these works, linear mixture MDPs~\citep{ICML'20:Ayoub-mixture} and linear MDPs~\citep{COLT'20:Jin-linear-mdp} are two of the most popular models. In this work, we focus on linear mixture MDPs whose transition is a linear mixture model.

The exploration of adversarial linear mixture MDPs remains an emerging area of research. In particular, \citet{ICML'20:OPPO} first study adversarial linear mixture MDPs with the unknown transition and full-information feedback. They propose a policy optimization algorithm OPPO that achieves $\Ot(dH^2 \sqrt{K})$ regret. The subsequent work by~\citet{AISTATS'22:optimal-adversarial-mixture} enhances the result to $\Ot(dH^{3/2} \sqrt{K})$ and shows it is minimax optimal. For the more challenging setting with unknown transition and bandit feedback,~\citet{ICLR'23:bandit-unknown-mixture} achieve an $\Ot(dS^2 \sqrt{K} + \sqrt{HSAK})$ regret, which exhibits a notable gap to the $\Omega(dH \sqrt{K} + \sqrt{HSAK})$ lower bound established therein.

In this work, we study adversarial linear mixture MDPs with bandit feedback and unknown transition. We strictly improve the result of~\citet{ICLR'23:bandit-unknown-mixture} and make a step towards closing the gap between the upper and lower bound. Specifically, we propose an algorithm that attains $\Ot(d\sqrt{HS^3K} + \sqrt{HSAK})$ regret, strictly improving the $\Ot(dS^2 \sqrt{K} + \sqrt{HSAK})$ regret of~\citet{ICLR'23:bandit-unknown-mixture} since $H \leq S$ by the layered MDP structure. As a byproduct, our result improves the best-known $\Ot(HS\sqrt{AK})$ regret of~\citet{ICML'20:bandit-unknown-chijin} for tabular MDPs when $d \leq \sqrt{HA/S}$. We note that though the dependence $S$ of our result is suboptimal, to the best of our knowledge, for this challenging unknown transition and bandit feedback setting, closing the gap regarding the dependence on $S$ for the tabular case is also an open problem. Table~\ref{tab:comparison} summarizes our result and previous related results.

Our algorithm is similar to that of~\citet{ICLR'23:bandit-unknown-mixture}: we first estimate the unknown transition parameter and construct corresponding confident sets. Then we apply Online Mirror Descent (OMD) over the occupancy measure space induced by the estimated transition. The most natural approach to estimate the unknown parameter is solving a linear regression problem with the visit statuses of the next states being the target. However, since the learner only visits one state in each step, the visit statuses across different states are no longer independent. This makes the key self-normalized concentration in~\citep[Theorem 1]{NIPS'11:AY-linear-bandits}, as restated in Lemma~\ref{lem:self-normalized}, not applicable. To address this issue,~\citet{ICLR'23:bandit-unknown-mixture} propose to leverage the transition information of only one state with the largest uncertainty. Though this technique effectively bypasses the issue of state correlation and serves as the first solution for this problem, unfortunately, it discards the visit information of other states, leading to a notable gap to the lower bound.

To enhance the utility of visitation data, we introduce a new least square estimator for the unknown transition parameter that leverages the visit information of \emph{all states}, as opposed to only a single state in~\citet{ICLR'23:bandit-unknown-mixture}. As stated before, the noises now are \emph{non-independent} across different states. We address this key challenge by introducing a new self-normalized concentration lemma tailored specifically to accommodate non-independent random noises. This lemma was originally proposed by~\citet{NIPS'22:Perivier-dynamic-assortment} for the \emph{dynamic assortment} problem, where a seller selects the subset of products to present to the customer who will then purchase at most one \emph{single} item. They use this lemma to manage the product correlations and we make adaptations to handle the state correlations. This enhancement empowers our algorithm to explore the orientations of every state simultaneously, distinguishing our method from the singular direction approach of~\citet{ICLR'23:bandit-unknown-mixture}, and resulting in a tighter bound. To the best of our knowledge, this is the first work that bridges the two distinct fields: dynamic assortment and RL theory. Our innovative use of techniques from dynamic assortment problems to mitigate estimation errors in RL theory is novel and may provide helpful insights for future research.

\vspace{1mm}
\textbf{Organization.}~~The rest of the paper is organized as follows. We first discuss the related work in Section~\ref{sec:related-work} and formulate the problem setup in Section~\ref{sec:setup}. We introduce the proposed algorithm in Section~\ref{sec:alg} and present the regret guarantee in Section~\ref{sec:regret}. Finally, We conclude the paper in Section~\ref{sec:conclusion}. Due to page limits, we defer all the proofs to the appendices.

\vspace{1mm}
\textbf{Notation.}~~We denote by $[n]$ the set $\{1, \ldots, n\}$ and use $\indicator\{\cdot\}$ to denote the indicator function. For a vector $x \in \R^d$ and a positive semi-definite matrix $\Sigma \in \R^{d \times d}$, let $\norm{x}_{\Sigma} = \sqrt{x^\top \Sigma x}$. Let $a \wedge b = \min\{a, b\}$ for all $a, b \in \R$. The $\Ot(\cdot)$-notation hides all logarithmic factors.

\newpage
\section{RELATED WORK}
\label{sec:related-work}

In this part, we review related works in the literature.

\paragraph{RL with adversarial losses.} Learning tabular RL with adversarial losses has been well-studied in the literature~\citep{MatOR'09:online-MDP, MatOR'09:MDP-Yu, NIPS'13:O-REPS, ICML'19:unknown-transition-Rosenberg, NIPS'19:bandit-unknown-Rosenberg, ICML'20:bandit-unknown-chijin, ICML'20:Shani-bandit-unknown, NIPS'21:Luo-adversarial-linear}. In general, these studies can be divided into two categories based on the type of the algorithm. The first category solves adversarial MDPs using policy-optimization-based methods. The pioneering works of~\citet{MatOR'09:online-MDP} and~\citet{MatOR'09:MDP-Yu} first study adversarial MDPs under the known transition and full-information setting. \citet{ICML'20:Shani-bandit-unknown} make the first step to study the more difficult unknown transition and bandit feedback setting and propose an algorithm that achieves an $\Ot(H^2 S \sqrt{A} K^{2/3})$ regret. The subsequent work by~\citet{NIPS'21:Luo-adversarial-linear} improves the result to $\Ot(H^2 S \sqrt{AK})$. The second category solves adversarial MDPs using occupancy-measure-based algorithms. For the known transition setting, \citet{NIPS'13:O-REPS} propose the O-REPS algorithm that achieves near-optimal regret for full-information and bandit feedback respectively. \citet{ICML'19:unknown-transition-Rosenberg} investigate the unknown transition but full-information setting. When the transition is unknown and only bandit feedback is available, \citet{NIPS'19:bandit-unknown-Rosenberg} propose an algorithm and prove it enjoys an $\Ot(HS \sqrt{AK} / \alpha)$ regret with an addition assumption that all states are reachable with probability $\alpha > 0$ for any policy. Without this assumption, the regret bound degenerates to $\Ot(H^{3/2}SA^{1/4}K^{3/4})$. Later,~\citet{ICML'20:bandit-unknown-chijin} achieve $\Ot(H \sqrt{SAK})$ regret without the assumption of~\citet{NIPS'19:bandit-unknown-Rosenberg}. Finally, we remark that the existing tightest lower bound of $\Omega(H \sqrt{SAK})$ is established by~\citet{NIPS'18:Q-learning} for the unknown transition and full-information feedback, which also serves as a lower bound for the bandit feedback directly. In this work, we study the most challenging setting where the transition is unknown and only bandit feedback is available. Moreover, our solution falls into the second category, i.e., the occupancy-measure-based method.

\paragraph{RL with linear function approximation.} To permit RL algorithm handling MDPs with large state and action space, a large body of literature considers solving MDPs with linear function approximation. In general, these studies can be categorized into three lines based on the specific assumption of the underlying MDP. The first line of work is according to the low Bellman-rank assumption~\citep{ICML'17:Jiang-low-rank, ICML'19:Du-low-rank}, which assumes a low-rank factorization of the Bellman error matrix. The second line of work focuses on the linear MDPs~\citep{ICML'19:Yang-linear-mdp, COLT'20:Jin-linear-mdp}, where the transition kernel and loss function are parameterized as a linear function of a feature mapping $\phi: \S \times \A \to \R^d$. The last line of work considers linear mixture/kernel MDPs~\citep{ICML'20:Ayoub-mixture, COLT'21:Zhou-mixture-minimax, COLT'23:Zhao-variance-mixture}, where the transition kernel can be parameterized as a linear function of a feature mapping $\phi: \S \times \A \times \S \to \R^d$. Note that all the above works focus on MDPs with with linear function approximation under the \emph{stochastic} loss functions. In this work, we investigate linear mixture MDPs but with the \emph{adversarial} loss functions.

\paragraph{RL with adversarial losses and linear function approximation.} Recent advances have emerged in learning adversarial RL with linear function approximation~\citep{NIPS'21:Neu-adversarial-linear, NIPS'23:Zhong-adversarial-linear, arXiv'23:Sherman-optimal-linear, NIPS'21:Luo-adversarial-linear, ICML'23:Dai-adversarial-linear, ICML'23:Sherman-improved-linear, TMLR'23:Kong-adversarial-linear, ICLR'24:liu-linear-bandit, ICML'20:OPPO, AISTATS'22:optimal-adversarial-mixture, NeurIPS'23:linearMDP,  AAAI'24:linearMixture, ICLR'23:bandit-unknown-mixture}. Generally, these studies can be divided into two lines. The first line focuses on the linear MDPs. \citet{NIPS'21:Neu-adversarial-linear} first study adversarial linear MDPs with bandit feedback but under the known transition setting.~\citet{NIPS'23:Zhong-adversarial-linear} first investigate the full-information and unknown transition setting and this setting is further studied by~\citet{arXiv'23:Sherman-optimal-linear} recently.~\citet{NIPS'21:Luo-adversarial-linear} make the first step to establish a sublinear regret for the more difficult unknown transition and bandit feedback setting. The result is further improved in~\citep{ICML'23:Dai-adversarial-linear, ICML'23:Sherman-improved-linear, TMLR'23:Kong-adversarial-linear, ICLR'24:liu-linear-bandit}. The second line of work considers the linear mixture MDPs. The seminal work of~\citet{ICML'20:OPPO} first studies adversarial linear mixture MDPs in the unknown transition and full-information feedback setting and proposes an optimistic proximal policy optimization algorithm. The subsequent work by~\citet{AISTATS'22:optimal-adversarial-mixture} improves their results to minimax optimality by using a weighted ridge regression and a Bernstein-type exploration bonus. The most recent work of~\citet{ICLR'24:Ji-horizon-free-mixture} studies the same setting and obtains a horizon-free regret which is independent of $H$ with the assumption that the losses are upper bounded by $1/H$. For the more challenging unknown transition and bandit feedback setting, the only existing work of~\citet{ICLR'23:bandit-unknown-mixture} achieves a regret of $\Ot(dS^2 \sqrt{K} + \sqrt{HSAK})$, which exhibits a gap compared to the $\Omega(dH \sqrt{K} + \sqrt{HSAK})$ lower bound established in their work. In our work, we consider the same unknown transition and bandit feedback setting as~\citet{ICLR'23:bandit-unknown-mixture} and improve the upper bound to $\Ot(d \sqrt{HS^3K} + \sqrt{HSAK})$, making a step towards closing the gap between the upper and lower bounds.

\section{PROBLEM SETUP}
\label{sec:setup}

In this section, we present the problem setup of episodic linear mixture MDPs with adversarial losses.

\paragraph{Episodic adversarial MDPs.} In this paper, we consider episodic adversarial MDP, which is denoted by a tuple $\M=(\S, \A, H, \{P_h\}_{h=1}^H, \{\ell_k\}_{k=1}^K)$. Here $\S$ is the state space with cardinality $|\S| = S$, $\A$ is the action space with cardinality $|\A| = A$, $H$ is the length of each episode, $K$ is the number of episodes, $P_h: \S \times \A \times \S \to [0, 1]$ is the transition kernel with $P_h(s' \given s, a)$ is being the probability of transiting to state $s'$ from state $s$ and taking action $a$ at stage $h$, $\ell_k: \S \times \A \to [0, 1]$ is the loss function, which may be chosen in an adversarial manner. Following previous studies~\citep{NIPS'13:O-REPS, ICML'19:unknown-transition-Rosenberg, ICML'20:bandit-unknown-chijin}, we assume the MDP has a layered structure, satisfying the conditions:
\begin{itemize}[leftmargin=*,topsep=0pt,parsep=0pt]
  \item The state space $\S$ consists of $H + 1$ disjoint layers such that $\S = \cup_{h=1}^{H+1} \S_h$ and $\S_i \cap \S_j = \emptyset$ for $i \neq j$.
  \item $\S_1 = \{s_1\}$ and $\S_{H+1} = \{s_{H+1}\}$ are singletons.
  \item Transition is possible only between adjacent layers, that is $P_h(s' \given s, a) = 0$ for all $s \in \S_h$ and $s' \notin \S_{h+1}$.
\end{itemize}

A policy $\pi = \{\pi_h\}_{h=1}^H$ is a collection of mapping $\pi_h$, where each $\pi_h: \S \to \Delta(\A)$ is a function maps a state $s$ to distributions over $\A$ at stage $h$. Define the expected loss of an policy $\pi$ at episode $k$ as
\begin{align}
  \label{eq:loss}
  L_k(\pi) =\mathbb{E}\left[\sum_{h=1}^H \ell_{k, h}\left(s_h, a_h\right) \givenn P, \pi \right],
\end{align}
where the expectation is taken over the randomness of the stochastic transition and policy.

In the \emph{online MDP} setting, the interaction protocol between the learner and the environment is given as follows. The interaction proceeds in $K$ episodes. At the beginning of episode $k$, the environment chooses a loss function $\ell_k$, which may be in an adversarial manner. Simultaneously, the learner chooses a policy $\pi_k = \{\pi_{k,h}\}_{h=1}^H$. At each stage $h \in [H]$, the learner observes the state $s_{k,h}$, chooses an action $a_{k,h}$ sampled from $\pi_{k,h}(\cdot \given s_{k,h})$, obtains reward $\ell_{k, h}(s_{k,h}, a_{k,h})$ and transits to the next state $s_{k, h+1}\sim P_h(\cdot \given s_{k,h}, a_{k,h})$. In this work, we consider the \emph{bandit feedback} setting where the learner can only observe the losses for the visited state-action pairs: $\{\ell_k(s_{k,h}, a_{k,h})\}_{h=1}^H$. The goal of the learner is to minimize regret, defined as
\begin{align}
  \label{eq:regret}
  \Reg(K) = \sum_{k=1}^K L_k(\pi_k) - \sum_{k=1}^K L_k(\pi^*),
\end{align}
where $\pi^* \in \argmin_{\pi \in \Pi}\sum_{k=1}^K L_k(\pi)$ is the optimal policy and $\Pi$ is the set of all stochastic policy.

\paragraph{Linear Mixture MDPs.} We focus on a special class of MDPs named \emph{linear mixture MDPs}~\citep{ICML'20:Ayoub-mixture, ICML'20:OPPO, COLT'21:Zhou-mixture-minimax, AISTATS'22:optimal-adversarial-mixture, NeurIPS'23:linearMDP}, where the transition kernel is linear in a known feature mapping $\phi: \S \times \A \times \S \to \R^d$ with the following definition.
\begin{myDef}[Linear Mixture MDPs]
  \label{def:mixture-MDPs}
  An MDP instance $\M=(\S, \A, H, \{P_h\}_{h=1}^H, \{\ell_k\}_{k=1}^K)$ is called an inhomogeneous, episodic $B$-bounded linear mixture MDP if there exist a \emph{known} feature mapping $\phi(s' \given s, a): \S \times \A \times \S \to \R^d$ with $\norm{\phi(s' \given s, a)}_2 \leq 1$ and \emph{unknown} vectors $\{\theta_h^*\}_{h=1}^H \in \R^d$ with $\norm{\theta_h^*}_2 \leq B$, such that for all $(s, a, s') \in \S \times \A \times \S$ and $h \in [H]$, it holds that $P_h(s' \given s, a) = \inner{\phi(s' \given s, a)}{\theta_h^*}$.
\end{myDef}

\paragraph{Occupancy measure.} Previous studies~\citep{NIPS'13:O-REPS,ICML'20:bandit-unknown-chijin} have shown the importance of the concept of \emph{occupancy measure} for solving adversarial MDPs via online learning techniques. Specifically, for some policy $\pi$ and transition kernel $P$, the occupancy measure $q^{P, \pi}$ is defined as the probability of visiting the state-action pair $(s, a)$ when executing policy $\pi$ under the transition $P$, that is
\begin{align}
  q^{P, \pi}(s, a) = \Pr[(s_h, a_h) = (s, a) \given P, \pi],
\end{align}
where $h = h(s)$ is the index of the layer that state $s$ belongs to. A valid occupancy measure $q$ satisfies the following two properties. First, according to the loop-free structure, each layer is visited once and only once, and thus for all $h \in [H]$, we have $\sum_{s \in \S_h} \sum_{a \in \A} q(s, a) = 1$. Second, the probability of entering a state when coming from a previous layer is equal to the probability of leaving the state when going to the next layer, that is for all $h = 2, \ldots, H$ and $s \in \S_h$, we have $\sum_{(s', a')\in \S_{h-1} \times \A} q(s', a')P_{h-1}(s \given s', a') = \sum_{a \in \A}q(s, a)$. Clearly, a valid occupancy measure $q$ induce a policy $\pi$ such that $\pi^q(a \given s) = q(s, a) / \sum_{a' \in \A} q(s, a')$. For a fixed transition kernel $P$, we denote by $\Delta(P)$ the set of all valid occupancy measures induced by $P$. Similarly, we denote by $\Delta(\P)$ the set of occupancy measures whose induced transition belongs to a set of transitions $\P$.

With the concept of occupancy measure, we can reduce this problem to the online linear optimization. Specifically, the expected loss of a policy $\pi$ at episode $k$ defined in~\eqref{eq:loss} can be rewritten as
\begin{align*}
  L_k(\pi) = \sum_{h=1}^H \sum_{s \in \S_h} \sum_{a \in \A} q^{P, \pi}(s, a) \ell_k(s, a) = \inner{q^{P, \pi}}{\ell_k}.
\end{align*}
Then the regret in~\eqref{eq:regret} can be rewritten as
\begin{align}
  \Reg(K) = \sum_{k=1}^K \inner{q^{P, \pi_k} - q^{P, \pi^*}}{\ell_k}.
\end{align}
We define $q^* \triangleq  q^{P, \pi^*} \in \Delta(P)$ to simplify the notation.

\section{THE PROPOSED ALGORITHM}
\label{sec:alg}

This section introduces our proposed \textsf{VLSUOB-REPS} algorithm (\underline{V}ector \underline{L}east \underline{S}quare \underline{U}pper \underline{O}ccupancy \underline{B}ound \underline{R}elative \underline{E}ntropy \underline{P}olicy \underline{S}earch) for adversarial linear mixture MDPs with unknown transition in the bandit feedback setting. \textsf{VLSUOB-REPS} consists of three key components: (\romannumeral1) estimating the unknown transition parameter and maintaining corresponding confidence set; (\romannumeral2) constructing loss estimators; and (\romannumeral3) applying online mirror descent over the occupancy measure space. We introduce the details below.

\subsection{Transition Estimator}
One of the main difficulties comes from the unknown transition kernel $P$. To address this issue, most existing works~\citep{ICML'20:Ayoub-mixture, ICML'20:OPPO, COLT'21:Zhou-mixture-minimax} use the method of \emph{value-targeted regression} (VTR) to learn the unknown parameter $\theta_h^*$  together with the corresponding confidence set. Specifically, for any function $V: \S \to \R$, define $\phi_{V}(s_{k,h}, a_{k,h}) = \sum_{s'}\phi(s' \given s_{k,h}, a_{k,h}) V(s')$. By the definition of linear mixture MDPs in Definition~\ref{def:mixture-MDPs}, we have
\begin{align*}
  P_h(\cdot \given s_{k,h}, a_{k,h})^\top V(\cdot) = \inner{\phi_{V}(s_{k,h}, a_{k,h})}{\theta_h^*}.
\end{align*}
Therefore, learning the underlying $\theta_h^*$ can be regarded as solving a ``linear bandit'' problem~\citep{book:bandit}, where the context is $\phi_{V}(s_{k,h}, a_{k,h})$, and the noise is $V(s_{k, h+1}) - P_h(\cdot \given s_{k,h}, a_{k,h})^\top V(\cdot)$. Thus, previous works~\citep{ICML'20:Ayoub-mixture,ICML'20:OPPO} set the estimator $\theta_{k,h}$ as the minimizer of the least squares linear regression objective:
\begin{align*}
  \sum_{i=1}^{k-1}\left[\phi_{V_{i, h+1}}(s_{i,h}, a_{i,h})^\top \theta -V_{i, h+1}(s_{i, h+1})\right]^2 + \lambda_k \|\theta\|_2^2,
\end{align*}
where $V_{k,h}$ is the state value function defined as $V_{k,h}(s) = \E[\sum_{h'=h}^H \ell_{k,h}(s_{k,h}, a_{k,h}) \given P, \pi_k, s_{k,h}=s]$.
A similar \emph{weighted} least squares linear regression method is used in the works~\citep{COLT'21:Zhou-mixture-minimax, AISTATS'22:optimal-adversarial-mixture}, which further utilizes the variance information of the value functions to gain a sharper confidence set.

Although value-targeted regression is the most popular method to estimate the unknown transition parameter in the literature, it is \emph{not} applicable in our setting. The reason is that this method can only guarantee $\hat{P}_h(\cdot \given s, a)^\top V_{k, h+1}(\cdot) \approx P_h(\cdot \given s, a)^\top V_{k, h+1}(\cdot)$, where $\hat{P}$ is the estimated transition kernel. This method learns the transition kernel implicitly and bypasses the need for fully estimating the transition, which can be viewed as ``model-free'' in this sense. However, it is not sufficient for our purpose since the occupancy measure also depends on the transition kernel, which requires us to learn the transition kernel explicitly and ensure the estimated transition is accurate enough, i.e., $\hat{P} \approx P$.

To this end, an alternative way to learn the unknown transition parameter is using the vanilla transition information directly. Specifically, denote $\Phi_{s, a} \in \R^{d \times S}$ with $\Phi_{s, a}(:, s') = \phi(s' \given s, a)$ and let $\delta_s \in \{0, 1\}^S$ be the one-hot vector with $\delta_s(s) = 1$. Then, we can rewrite the transition kernel as $P_h(\cdot \given s, a) = \Phi_{s, a}^\top \theta_h^*$. Thus, to learn the unknown parameter $\theta_h^*$, we consider using $\Phi_{s_{k,h}, a_{k,h}}$ as feature and $\delta_{s_{k, h+1}}$ as the regression target. Then, the estimator $\theta_{k,h}$ is defined as the solution of the following linear regression problem:
\begin{align}
  \label{eq:estimator}
  \theta_{k, h}=\argmin_{\theta \in \R^d} \sum_{i=1}^{k-1}\left\|\Phi_{s_{i, h}, a_{i, h}}^{\top} \theta-\delta_{s_{i, h+1}}\right\|_2^2+\lambda_k\|\theta\|_2^2.
\end{align}
The closed-form solution is $\theta_{k, h} = \Lambda_{k,h}^{-1} b_{k,h}$ with
\begin{align}
  \Lambda_{k,h} & = \sum_{i=1}^{k-1}\sum_{s' \in \S_{h+1}} \phi(s' | s_{i,h}, a_{i,h})\phi(s' | s_{i,h}, a_{i,h})^\top + \lambda_k I_d, \nonumber \\
  b_{k,h}       & = \sum_{i=1}^{k-1}\sum_{s' \in \S_{h+1}} \delta_{s_{i, h+1}}(s') \phi(s' \given s_{i,h}, a_{i,h}). \label{eq:estimator-b}
\end{align}

Nonetheless, a significant challenge remains to be solved. Specifically, let $\varepsilon_{i,h} = P_h(\cdot \given s_{i,h}, a_{i,h}) - \delta_{s_{i, h+1}}$ be the noise at episode $i$ at stage $h$ due to the transition. It is clear that $\varepsilon_{i,h} \in [-1, 1]^S$, $\E_{i,h}[\varepsilon_{i,h}] = \mathbf{0}$. One may consider establishing an ellipsoid confidence set for $\theta_h^*$ by applying the self-normalized concentration for vector-valued martingales~\citep[Theorem 1]{NIPS'11:AY-linear-bandits}, as restated in Lemma~\ref{lem:self-normalized} of Appendix~\ref{appendix:supporting-lemmas}. However, since the learner only transits to one state in each layer, the noises across different states are \emph{no longer independent}. Concretely, it hold that $\sum_{s \in \S} \varepsilon_{i,h}(s) = 0$. Thus the noises $\varepsilon_{i,h}(s)$ of different states are \mbox{$1$-subgaussian} but they are not independent. This fact makes the key self-normalized concentration in Lemma~\ref{lem:self-normalized} no longer applicable.

To address this challenge,~\citet{ICLR'23:bandit-unknown-mixture} propose to use the transition information of \emph{only one} certain state $s'_{i, h+1}$ in the next layer, which they call the \emph{imaginary} next state. They set the estimator $\theta_{k,h}$ as the minimizer of the following linear regression problem:
\begin{align}
  \label{eq:estimator-one}
  \sum_{i=1}^{k-1}\left[\phi(s'_{i,h+1} | s_{i,h}, a_{i,h})^\top \theta-\delta_{s_{i, h+1}}(s'_{i, h+1})\right]^2 + \lambda_k\|\theta\|_2^2.
\end{align}
Note that the imaginary next state $s'_{i, h+1}$ is not the actual next state $s_{i, h+1}$ experienced by the learner. Instead, they choose the imaginary next state $s'_{i, h+1}$ as the state with the largest uncertainty, formally,
\begin{align*}
  s'_{i, h+1} = \argmax_{s \in \S} \norm{\phi(s \given s_{i,h}, a_{i,h})}_{M_{i,h}^{-1}},
\end{align*}
where $M_{i,h}$ is the feature covariance matrix, set as
\begin{align*}
  \sum_{j=1}^{i-1} \phi(s'_{j, h+1} \given s_{j,h}, a_{j,h})\phi(s'_{j, h+1} \given s_{j,h}, a_{j,h})^\top + \lambda_i I.
\end{align*}
With this choice, they can control the uncertainty of other states by that of the imaginary next state.

Though the method of using one state at each stage in~\citet{ICLR'23:bandit-unknown-mixture} is novel and provides an initial solution for this problem, it discards the visit information of other states and leads to a notable gap to the lower bound. To fully utilize the visit information, we use the information of \emph{all} states and construct the estimator as in~\eqref{eq:estimator}, instead of the only one state in~\eqref{eq:estimator-one}. To address the non-independent noise issue, we introduce a new self-normalized concentration lemma tailored specifically for non-independent random noises. This lemma was originally proposed by~\citet{NIPS'22:Perivier-dynamic-assortment} for the \emph{dynamic assortment} problem, where a seller selects the subset of products to present to the customer who will then purchase one \emph{single} item. They also face the non-independent random noises issue as the customer will only purchase at most \emph{one} product, which is similar to our problem where the learner will only visit \emph{one} state. Thus, they use this lemma to manage the product correlations and we make adaptations to handle the state correlations. Differently,~\citet[Theorem C.6]{NIPS'22:Perivier-dynamic-assortment} establish a variance-aware concentration inequality. In our work, we adapt their inequality into a simplified variance-independent form, which is well-suited for our analytical needs.

\begin{myLemma}
  \label{lem:concentration}
  Let $\{\F_t\}_{t=0}^\infty$ be a filtration. Let $\{\delta_t\}_{t=1}^\infty$ be an $\R^N$-valued stochastic process such that $\delta_t$ is $\F_t$-measurable one-hot vector. Furthermore, assume $\E[\delta_t | \F_{t-1}] = p_t$ and define $\varepsilon_t = p_t - \delta_t$. Let $\{x_t\}_{t=1}^\infty$ be a sequence of $\R^{N \times d}$-valued stochastic process such that $x_t$ is $\F_{t-1}$-measurable and $\norm{x_{t, i}}_2 \leq 1, \forall i \in [N]$. Let $\{\lambda_t\}_{t=1}^\infty$ be a sequence of non-negative scalars. Define
  \begin{align*}
    Y_t = \sum_{i=1}^t \sum_{j=1}^N x_{i,j} x_{i,j}^\top + \lambda_t I_d, \quad S_t = \sum_{i=1}^t \sum_{j=1}^N \varepsilon_{i,j} x_{i,j}.
  \end{align*}
  Then, for any $\zeta \in (0, 1)$, with probability at least $1 - \zeta$, we have for all $t \geq 1$,
  \begin{align*}
    \norm{S_t}_{Y_t^{-1}} \leq \frac{\sqrt{\lambda_t}}{4}+\frac{4}{\sqrt{\lambda_t}} \log \left(\frac{2^d \det\left(Y_t\right)^{\frac{1}{2}} \lambda_t^{-\frac{d}{2}}}{\zeta}\right).
  \end{align*}
\end{myLemma}

With the above lemma, we can build the confidence set for the unknown parameter $\theta_h^*$ as follows.
\begin{myLemma}
  \label{lem:confidence-set}
  Let $\zeta \in (0, 1)$, then for any $k \in [K]$ and simultaneously for all $h \in [H]$, with probability at least $1 - \zeta$, it holds that
  $$\theta_h^* \in \C_{k,h} \mbox{ where } \C_{k,h} = \{\theta \in \R^d ~|~ \norm{\theta - \theta_{k,h}}_{\Lambda_{k,h}} \leq \beta_{k}\}$$
  with $\beta_{k} = (B+\frac{1}{4}) \sqrt{\lambda_k} + \frac{2}{\sqrt{\lambda_k}}(2\log(\frac{H}{\zeta}) + d \log(4 + \frac{4Sk}{\lambda_k d}))$.
\end{myLemma}

\begin{myRemark}
  Compared with the confidence set of $\norm{\theta - \theta_{k,h}}_{M_{k,h}} \leq \beta_{k} $ of~\citet{ICLR'23:bandit-unknown-mixture}, our confidence set in Lemma~\ref{lem:confidence-set} is tighter since $M_{k,h} \preceq \Lambda_{k,h}$. A primary challenge in constructing such a confidence set is bounding the self-normalized concentration term $\norm{\sum_{i=1}^k\sum_{s' \in \S_{h+1}}\varepsilon_{i,h+1}(s') \phi(s' \given s_{i,h}, a_{i,h})}_{\Lambda_{k,h}^{-1}}$. Due to the non-independent noises, we can not apply the self-normalized concentration in Lemma~\ref{lem:self-normalized} directly. As a solution,~\citet{ICLR'23:bandit-unknown-mixture} propose to concentrate on a singular state per layer, which only need to bound the term $\norm{\sum_{i=1}^k \varepsilon_{i,h+1}(s'_{i, h+1}) \phi(s'_{i, h+1} \given s_{i,h}, a_{i,h})}_{M_{k,h}^{-1}}$. While this approach bypasses the complications introduced by non-independent noises, it discards the visit information of other states. In contrast, we bound this challenging term by Lemma~\ref{lem:concentration}. This allows us to utilize the information of all states, leading to an improved bound. Intuitively, this new concentration lemma empowers our algorithm to explore the orientations of every state simultaneously, as opposed to the singular direction approach in~\citet{ICLR'23:bandit-unknown-mixture}. \endenv
\end{myRemark}

Based on the above lemma, we can construct the confidence set $\P_k = \{\P_{k,h}\}_{h=1}^H$ for the transition $P$ as
\begin{align}
  \label{eq:confidence-set}
  \P_{k,h} = \{\Ph_h \thinspace | \thinspace \exists \theta \in \C_{k,h}, \Ph_h(s' | s, a) = \phi(s' | s, a)^\top \theta\}
\end{align}
for all $(s, a, s') \in \S \times \A \times \S$. According to Lemma~\ref{lem:confidence-set}, we have $P \in \P_k$ with probability at least $1 - \zeta$.

\subsection{Loss Estimator}

A common technique to deal with the bandit-feedback setting is to construct a loss estimator $\ellh_{k,h}$ for the true loss function $\ell_{k,h}$ based on historical observations. When the transition is known, existing works~\citep{NIPS'13:O-REPS,NIPS'19:bandit-unknown-Rosenberg} construct the unbiased estimator as
\begin{align}
  \label{eq:loss-estimator-unbiased}
  \ellh_{k}(s, a) = \frac{\ell_{k}(s,a)}{q^{P, \pi_k}(s, a)} \indicator_k(s, a),
\end{align}
where $\indicator_k(s, a) = 1$ if $(s, a)$ is visited at episode $k$ and $\indicator_k(s, a) = 0$ otherwise. However, this method can not be directly applied to the unknown transition setting since the occupancy measure $q^{P, \pi_k}$ is unknown.~\citet{NIPS'19:bandit-unknown-Rosenberg} directly use the empirical occupancy measure $\hat{q}^{P, \pi_k}$ in place of $q^{P, \pi_k}$ to construct the estimator that could be either an overestimate or an underestimate, leading a loose regret bound.

To address this issue,~\citet{ICML'20:bandit-unknown-chijin} follow the principle of ``optimistic in the face of uncertainty'' and builds an underestimate for the loss function $\ell_{k,h}$ to encourage exploration. Since the true transition $P$ belongs to the confidence set $\P_k$ with high probability. To build an underestimate for $\ell_{k,h}$, they propose to replace $q^{P, \pi_k}(s, a)$ in~\eqref{eq:loss-estimator-unbiased} with its \emph{upper occupancy bound}, defined as the largest possible probability of visiting $(s, a)$ under the confidence set $\P_k$. Formally,
\begin{align}
  \label{eq:upper-confidence-bound}
  u_k(s, a) = \max_{\Ph \in \P_k} q^{\Ph, \pi_k}(s, a).
\end{align}
The above step can be computed efficiently by the \mbox{\textsc{Comp-UOB}} procedure of~\citet{ICML'20:bandit-unknown-chijin}. Additionally, they adopt the idea of \emph{implicit exploration} of~\citet{NIPS'15:Neu-implicit} to further increase the denominator by some fixed amount $\gamma > 0$, which is for several technical reasons such as obtaining a high probability bound. Finally, the estimator is built as
\begin{align}
  \label{eq:loss-estimator}
  \ellh_{k}(s, a) = \frac{\ell_{k}(s,a)}{u_k(s, a) + \gamma} \indicator_k(s, a).
\end{align}
Clearly, $\ellh_{k}(s, a)$ is an underestimate of $\ell_{k}(s,a)$ since $u_k(s, a) \geq q^{P, \pi_k}(s, a)$ with high probability.

In our algorithm, we follow the work of~\citet{ICML'20:bandit-unknown-chijin} and employ the loss estimator defined in~\eqref{eq:loss-estimator}.

\subsection{Online Mirror Descent}

Online Mirror Descent (OMD) is a powerful framework for solving online convex optimization problems~\citep{book'19:Orabona-OL, COLT'18:Wei-adaptive-bandit, arXiv'21:Sword++}. As discussed in Section~\ref{sec:setup}, our problem is closely related to the online linear optimization problem over the occupancy measure space. Thus, we utilize OMD as a key component of our algorithm. We apply OMD over the occupancy measure space $\Delta(\P_k)$ induced by the confidence set $\P_k$. Specifically, we update the occupancy measure as follows:
\begin{align}
  \label{eq:omd}
  \qh_{k+1} = \argmin_{q \in \Delta(\P_{k})} \eta \inner{q}{\ellh_{k}} + \Dp(q \| \qh_k),
\end{align}
where $\ellh_k$ is the loss estimator defined in~\eqref{eq:loss-estimator}, $\eta > 0$ is step size, $\psi(q) = \sum_{s,a} q(s, a) \log q(s, a) - \sum_{s, a}q(s, a)$ is the unnormalized negative entropy, and $\Dp(q \| q') = \sum_{s, a}q(s, a) \log \frac{q(s, a)}{q'(s, a)} - \sum_{s, a}(q(s, a) - q'(s, a))$ is the unnormalized KL-divergence. The update procedure in~\eqref{eq:omd} can also be implemented efficiently, as discussed in Appendix E of~\citet{ICLR'23:bandit-unknown-mixture}.

The detailed algorithm is presented in Algorithm~\ref{alg}. Line~\ref{line:estimator-start} - \ref{line:estimator-end} estimate the unknown transition, Line~\ref{line:u_k} compute the upper occupancy bound $u_k$, Line~\ref{line:loss-estimator} constructs the loss estimator $\ellh_k$, and Line~\ref{line:omd} runs OMD to update the occupancy measure $\qh_{k+1}$. The learner execute the policy $\pi_{k+1}$ induced by $\qh_{k+1}$ in Line~\ref{line:policy}.

\begin{algorithm}[t]
  \caption{\textsf{VLSUOB-REPS}}
  \label{alg}
  \begin{algorithmic}[1]
    \REQUIRE Confidence parameter $\zeta$, step size $\eta$, regularization parameter $\lambda_k$, exploration parameter $\gamma$.
    \STATE \textbf{Initialization: } Set confidence set $\P_1$ as all transition kernels. For all $h \in [H]$ and all $s \in \S_h$, set $\qh_1(s, a) = \frac{1}{S_h \times A}$. Let $\pi_1 = \pi^{\qh_1}$, $\Lambda_{1, h} = \lambda_1 I_d, \forall h$.
    \FOR{$k = 1, \ldots, K$}
    \FOR{$h = 1, \ldots, H$} \label{line:estimator-start}
    \STATE Take action $a_{k,h} \sim \pi_{k,h}(\cdot \given s_{k,h})$.
    \STATE Suffer and observe loss $\ell_{k,h}(s_{k,h}, a_{k,h})$.
    \STATE Transit to $s_{k, h+1} \sim P_h(\cdot \given s_{k,h}, a_{k,h})$.
    \STATE $\theta_{k, h} = \Lambda_{k,h}^{-1} b_{k,h}$ with $\Lambda_{k,h}$ and $b_{k,h}$ as in~\eqref{eq:estimator-b}.
    \STATE Construct the confidence set as in~\eqref{eq:confidence-set}.
    \ENDFOR \label{line:estimator-end}
    \STATE Compute upper bound $u_k(s, a)$ as in~\eqref{eq:upper-confidence-bound}. \label{line:u_k}
    \STATE Construct loss estimator $\ellh_{k,h}$ as in~\eqref{eq:loss-estimator}. \label{line:loss-estimator}
    \STATE Compute occupancy measure $\qh_{k+1}$ as in~\eqref{eq:omd}. \label{line:omd}
    \STATE Update policy $\pi_{k+1} = \pi^{\qh_{k+1}}$. \label{line:policy}
    \ENDFOR
  \end{algorithmic}
\end{algorithm}

\section{REGRET GUARANTEE}
\label{sec:regret}

In this section, we present the regret upper bound of our algorithm and the proof sketch.

\subsection{Regret Upper Bound}
The regret bound of our algorithm \textsf{VLSUOB-REPS} is guaranteed by the following theorem.
\begin{myThm}
  \label{thm:regret-bound}
  Set the step size $\eta$ and exploration parameter $\gamma$ as $\eta = \gamma = \sqrt{\frac{H \log(HSA/\zeta)}{KSA}}$, the regularization parameter $\lambda_k$ as $\lambda_1 = 1, \lambda_k = d\log(kS)$, $\forall k > 1$. With probability at least $1 - 5\zeta$, \textsf{VLSUOB-REPS} algorithm ensures the regret $\Reg(K)$ is upper bounded by
  \begin{align*}
    \O\Bigg(d  \sqrt{HS^3K} \log^2 \Big(\frac{dSK}{\zeta}\Big) +\sqrt{H S A K \log \Big(\frac{H S A}{\zeta}\Big)}\Bigg).
  \end{align*}
\end{myThm}

\begin{myRemark}
  Compared with the regret bound of $\Ot(d S^2 \sqrt{K}+\sqrt{H S A K})$ in~\citet[Theorem 1]{ICLR'23:bandit-unknown-mixture}, our bound is better since $H \leq S$ by the layered structure of MDPs. As a byproduct, our result improves the best-known $\Ot(HS\sqrt{AK})$ regret for tabular MDPs~\citep{ICML'20:bandit-unknown-chijin} when $d \leq \sqrt{HA/S}$. \endenv
\end{myRemark}

\begin{myRemark}
  Compared with the lower bound of $\Omega(d H \sqrt{K}+\sqrt{H S A K})$ established in~\citet[Theorem 2]{ICLR'23:bandit-unknown-mixture}, our regret is suboptimal in the dependence on $S$. However, note that the dependence on $S$ remains suboptimal even for tabular MDPs~\citep{ICML'20:bandit-unknown-chijin}. How to close this gap is an important open question and we leave it as future work. \endenv
\end{myRemark}

\subsection{Occupancy Measure Difference}

In this part, we introduce a key technical lemma that bounds the occupancy measure difference induced by the different transitions in the confidence set and is critical in our analysis.

\begin{myLemma}[Occupancy measure difference for linear mixture MDPs]
  \label{lem:occupancy-measure-difference}
  For any collection of transition kernels $\{P_k^s\}_{s \in \S}$ such that $P_k^s \in \P_k$ for all $s \in \S$, if $\lambda \geq \delta$, with probability at least $1 - 2 \zeta$, it holds that
  \begin{align*}
    \sum_{k=1}^K \Big\|q^{P_k^s, \pi_k}-q_k\Big\|_1 \leq \O\left(d  \sqrt{H S^3 K} \log^2 ({d SK}/{\zeta})\right).
  \end{align*}
\end{myLemma}
\begin{myRemark}
  Compared with the occupancy measure difference $\Ot(d S^2 \sqrt{K})$ of~\citet[Lemma 2]{ICLR'23:bandit-unknown-mixture}, our bound $\Ot(d  \sqrt{HS^3K})$ in Lemma~\ref{lem:occupancy-measure-difference} is better since $H \leq S$ by the layed structure of MDPs. This improvement comes from the new self-normalized concentration for non-independent random noises in Lemma~\ref{lem:concentration}, which allows us to use the transition information of all states instead of only one as in~\citet{ICLR'23:bandit-unknown-mixture}. \endenv
\end{myRemark}
\begin{myRemark}
  Compared with the occupancy measure difference $\Ot(HS \sqrt{AK})$ of~\citet[Lemma 4]{ICML'20:bandit-unknown-chijin}, our bound $\Ot(d  \sqrt{HS^3 K})$ in Lemma~\ref{lem:occupancy-measure-difference} is better when $d \leq \sqrt{HA/S}$. Though our bound gets rid of the dependence on the action space size $A$, it still keeps the dependence on the state space $S$. As pointed by~\citet{ICLR'23:bandit-unknown-mixture}, the main hardness of simultaneously eliminating the dependence of the occupancy measure difference on both $S$ and $A$ is that though the transition kernel $P$ of a linear mixture MDP admits a linear structure, the occupancy measure still has a complicated recursive form: $q_k(s, a)=\pi_k(a | s)\langle{\theta}_{h(s)-1}^*, \sum_{\left(s^{\prime}, a^{\prime}\right) \in \mathcal{S}_{h(s)-1} \times \mathcal{A}} q_k\left(s^{\prime}, a^{\prime}\right) \phi\left(s | s^{\prime}, a^{\prime}\right)\rangle$. We leave the question of whether it is possible to eliminate the dependence on $S$ as future work. \endenv
\end{myRemark}

Next, we present the proof sketch of Lemma~\ref{lem:occupancy-measure-difference}.
\begin{myProofSketch}[of Lemma~\ref{lem:occupancy-measure-difference}]
  First, we introduce the following lemma, which bounds the error between the transition $P_h$ and the estimated transition $\Ph_{k,h} \in \P_{k,h}$.

  \begin{myLemma}
    \label{lem:transition-difference}
    Let $\Ph_k = \{\Ph_{k,h}\}_{h=1}^H$  with $\Ph_{k,h} \in \P_{k,h}$ such that $\Ph_{k,h}(s' \given s, a) = \phi(s' \given s, a)^\top \thetah_{k,h}$ for all $(s, a, s') \in \S_h \times \A \times \S_{h+1}$ for some $\thetah_{k,h} \in \C_{k,h}$. Then for any $\zeta \in (0, 1)$ and simultaneously for all $k \in [K]$ and $h \in [H]$, with probability at least $1 - \zeta$, it holds that
    \begin{align*}
      |\Ph_{k,h}(s' | s, a) - P_h(s' | s, a)| \leq 1 \wedge \beta_{k}\norm{\phi(s' \given s, a)}_{\Lambda_{k,h}^{-1}},
    \end{align*}
    where $\beta_{k}$ is the diameter of confidence set defined in Lemma~\ref{lem:confidence-set} and $\Lambda_{k,h}$ is covariance matrix defined in~\eqref{eq:estimator-b}.
  \end{myLemma}

  Then, we present the following lemma, which bounds the error between the occupancy measure by the estimation error of the transition. For the sake of brevity, we define $\epsilon_{k,h}(s' \given s, a) = 1 \wedge \beta_{k}\norm{\phi(s' \given s, a)}_{\Lambda_{k,h}^{-1}}$, then we have $|\Ph_{k,h}(s' | s, a) - P_h(s' | s, a)| \leq \epsilon_{k,h}(s' \given s, a)$ by Lemma~\ref{lem:transition-difference}. Then we have the following lemma.

  \begin{myLemma}
    \label{lem:occupancy-measure-difference-intermediate}
    For any collection of transition kernels $\{P_k^s\}_{s \in \S}$ such that $P_k^s \in \P_k$ for all $s \in \S$, if $\lambda \geq \zeta$, with probability at least $1 - 2 \zeta$, it holds that
    \begin{align*}
       & \sum_{k=1}^K \Big\|q^{P_k^s, \pi_k}-q_k\Big\|_1  \leq 2 S \sum_{k, h, s'} \epsilon_{k,h}^{s'} + 4 \beta_K S^2 \log \Big(\frac{H}{\zeta} \Big).
    \end{align*}
    Here $\sum_{k, h, s'}\epsilon_{k,h}^{s'} \triangleq \sum_{k,h} \sum_{s' \in \S_{h+1}}\epsilon_{k,h}(s' | s_{k,h}, a_{k,h})$.
  \end{myLemma}

  According to Lemma~\ref{lem:occupancy-measure-difference-intermediate}, to bound the occupancy measure difference, it suffices to bound the estimation error of the transition, that is the cumulative error of $\epsilon_{k,h}(s' \given s_{k,h}, a_{k,h})$ over all episodes $k$ and all stages $h$. Bounding this term is the main difference between our work and~\citet{ICLR'23:bandit-unknown-mixture}. In particular,~\citet{ICLR'23:bandit-unknown-mixture} only use the transition information of one state $s'_{k,h+1}$ with the maximum uncertainty in the next layer. Thus they can bound the estimation error of other states by that of $s'_{k,h+1}$. Specifically, note that $\epsilon_{k,h}(s' \given s, a) \leq \beta_{k}(1 \wedge \norm{\phi(s' \given s, a)}_{\Lambda_{k,h}^{-1}})$, they bound the cumulative error of transition as below:
  \begin{align*}
            & \sumk \sumh \sum_{s'\in \S_{h+1}} \beta_{k}(1 \wedge \norm{\phi(s' \given s_{k,h}, a_{k,h})}_{M_{k,h}^{-1}})  \\
    \leq \  & \beta_K \sumk \sumh S_{h+1} (1 \wedge \norm{\phi(s'_{k, h+1} \given s_{k,h}, a_{k,h})}_{M_{k,h}^{-1}})        \\
    \leq \  & \beta_K S \sqrt{K \sumk \big(1 \wedge \norm{\phi(s'_{k, h+1} \given s_{k,h}, a_{k,h})}_{M_{k,h}^{-1}}\big)^2} \\
    \leq \  & \Ot \big(\beta_K S \sqrt{d K} \big).
  \end{align*}
  Here the first inequality holds by the choice that $s'_{i, h+1} = \argmax_{s \in \S} \norm{\phi(s \given s_{i,h}, a_{i,h})}_{M_{i,h}^{-1}}$, the second inequality holds by the Cauchy-Schwarz inequality, the last inequality follows from the self-normalized concentration of~\citet[Lemma 10]{NIPS'11:AY-linear-bandits}.

  Instead, we use the transition information of all states and thus can bound this term directly as follows
  \begin{align*}
            & \sumk \sumh \sum_{s'\in \S_{h+1}} \beta_{k}(1 \wedge \norm{\phi(s' \given s_{k,h}, a_{k,h})}_{\Lambda_{k,h}^{-1}})                 \\
    \leq \  & \beta_K \sumh \sqrt{K S_{h+1} \sumk \sum_{s'\in \S_{h+1}}(1 \wedge \norm{\phi(s' \given s_{k,h} a_{k,h})}_{\Lambda_{k,h}^{-1}})^2} \\
    \leq \  & \Ot \big(\beta_K  \sqrt{d H SK} \big),
  \end{align*}
  which replaces a dependence on $\sqrt{S}$ with $\sqrt{H}$, resulting the final improvement over~\citet{ICLR'23:bandit-unknown-mixture}. \qed
\end{myProofSketch}

\subsection{Proof Sketch of Theorem~\ref{thm:regret-bound}}
\label{sec:proof-sketch}

Finally, we present the proof sketch of Theorem~\ref{thm:regret-bound}.

Define the occupancy measure under the true transition $P$ and policy $\pi_k$ as $q_k = q^{P, \pi_k}$. Then, the regret can be written as $\Reg = \sumk \inner{q_k - q^*}{\ell_k}$. Following the work of~\citet{ICML'20:bandit-unknown-chijin}, we decompose the regret as the following four terms:
\begin{align*}
  \Reg(K) \leq & \underbrace{\sumk \inner{\qh_k - q^*}{\ellh_k}}_{\regret} + \underbrace{\sumk \inner{q_k - \qh_k}{\ell_k}}_{\error}            \\
               & + \underbrace{\sumk \inner{\qh_k}{\ell_k - \ellh_k}}_{\biasone} + \underbrace{\sumk \inner{q^*}{\ellh_k - \ell_k}}_{\biastwo}.
\end{align*}
Here, the first $\regret$ term is the regret of the corresponding online linear optimization problem with respect to the loss estimator $\ellh_k$, which can be controlled by OMD via standard analysis and can be bounded by $\O(\sqrt{H S A K \log (H S A / \zeta)}+H \log (H / \zeta))$. The last $\biastwo$ term measures the bias of the loss estimator $\ellh_k$ with respect to the true loss $\ell_k$, which can be bounded by $\O(\sqrt{H S A K \log (S A / \zeta)})$ by the concentration of the implicit exploration~\citep{NIPS'15:Neu-implicit}. Finally, the remaining two terms $\error$ and $\biasone$ come from the error of using $\qh_k$ and $u_k$ to approximate $q_k$ respectively, which are closely related to the occupancy measure difference in Lemma~\ref{lem:occupancy-measure-difference}. We bound these two terms in the rest. Bounding $\regret$ and $\biastwo$ is relatively standard and we defer the proofs to Appendix~\ref{appendix:proof-regret-bound}.

\textbf{Bounding} $\error$ \textbf{Term.} With Lemma~\ref{lem:occupancy-measure-difference}, we immediately obtain the following bound on $\error$.

\begin{myLemma}
  \label{lem:error}
  For any $\zeta \in (0, 1)$, with probability at least $1 - 2\zeta$, \textsf{VLSUOB-REPS} algorithm ensures that
  \begin{align*}
    \error \leq \O\Big(d\sqrt{HS^3K}\log^2 (dKS/\zeta)\Big).
  \end{align*}
\end{myLemma}

\begin{proof}
  Let $P_k^s = P^{\qh_k} \in \P_k$ for all $s$ such that $\qh_k = q^{P_k, \pi_k}$. Since $\ell_k(s, a) \in [0, 1]$ for all $k \in [K]$ and $(s, a) \in \S \times \A$, we have $\error \leq \sumk \sum_{s,a} \abs{\qh_k(s, a) - q_k(s, a)} = \sumk \sum_{s,a} \abs{q^{P_k^s, \pi_k} - q_k(s, a)}$. The proof is then completed by applying Lemma~\ref{lem:occupancy-measure-difference}.
\end{proof}

\textbf{Bounding} $\biasone$ \textbf{Term.} To bound this term, we need to show the loss estimator $\ellh_k$ is close to the true loss function $\ell_k$. This is guaranteed by the fact that the confidence set becomes more and more accurate for frequently visited state-action pairs.

\begin{myLemma}
  \label{lem:bias1}
  For any $\zeta \in (0, 1)$, with probability at least $1 - 3\zeta$, \textsf{VLSUOB-REPS} algorithm ensures that
  \begin{align*}
    \biasone \leq \O\left(d\sqrt{HS^3K}\log^2 (dKS/\zeta) + \gamma SAK\right).
  \end{align*}
\end{myLemma}

\begin{proof}
  To bound the $\biasone$, we first bound the term $\sumk \inner{\qh_k}{\ell_k - \E_{k-1, H}[\ellh_k]}$ as follows.
  \begin{align*}
         & \sum_k\left\langle\widehat{q}_k, \ell_k-\mathbb{E}_{k-1,H}\left[\widehat{\ell}_k\right]\right\rangle                                       \\
    =    & \sum_{k, s, a} \widehat{q}_k(s, a) \ell_k(s, a)\left(1-\frac{\mathbb{E}_{k-1,H}\left[\mathbb{I}_k\{s, a\}\right]}{u_k(s, a)+\gamma}\right) \\
    =    & \sum_{k, s, a} \widehat{q}_k(s, a) \ell_k(s, a)\left(1-\frac{q_k(s, a)}{u_k(s, a)+\gamma}\right)                                           \\
    \leq & \sum_{k, s, a}\left|u_k(s, a)-q_k(s, a)\right|+\gamma S A K .
  \end{align*}
  Since $u_k = q^{P_k^s, \pi_k}$ where $P_k^s = \argmax_{\Ph \in \P_k} q^{\Ph, \pi_k}(s)$, the term $\sum_{k, s, a}\left|u_k(s, a)-q_k(s, a)\right|$ can be controlled by Lemma~\ref{lem:occupancy-measure-difference} again. It remains to bound the term $\sumk \inner{\qh_k}{\E_{k-1, H}[\ellh_k] - \ellh_k}$. Since the fact that $P^{\qh_k} \in \P_k$ and $u_k(s, a) = \max_{\Ph \in \P_k} q^{\Ph, \pi_k}(s, a)$, we have $\sum_{s, a} \widehat{q}_k(s, a) \widehat{\ell}_k(s, a) \leq \sum_{s, a} u_k(s, a) \widehat{\ell}_k(s, a) = H$. Then using the Azuma-Hoeffding inequality, with probability at least $1 - \zeta$, we have
  \begin{align*}
    \sumk \inner{\qh_k}{\E_{k-1, H}[\ellh_k] - \ellh_k} \leq H \sqrt{2K\log\left({1}/{\zeta}\right)}.
  \end{align*}
  Applying the union bound finishes the proof.
\end{proof}

\section{CONCLUSION}
\label{sec:conclusion}

In this work, we consider learning adversarial linear mixture MDPs with bandit feedback and unknown transition. We propose a new algorithm that achieves an $\Ot(d\sqrt{HS^3K} + \sqrt{HSAK})$ regret with high probability. Our result strictly improves the previously best-known $\Ot(dS^2 \sqrt{K} + \sqrt{HSAK})$ regret~\citep{ICLR'23:bandit-unknown-mixture}. As a byproduct, it improves the best-known $\Ot(HS\sqrt{AK})$ result for tabular MDPs~\citep{ICML'20:bandit-unknown-chijin} when $d \leq \sqrt{HA/S}$. To achieve this result, we first propose a new least square estimator for the unknown transition that leverages the visit information of all states, as opposed to only a single state in~\citet{ICLR'23:bandit-unknown-mixture}. Then we introduce a new self-normalized concentration designed specifically for non-independent noises to handle the state correlations.

Several questions remain open for future study. First, the dependence on $S$ of our result is suboptimal, and how to close this gap is an important open problem. Moreover, optimizing the dynamic regret of adversarial MDPs is an emerging direction to facilitate algorithms with more robustness in non-stationary environments. Recent literature has explored the dynamic regret of adversarial MDPs with full-information feedback~\citep{NIPS'20:Fei-dynamic, ICML'22:mdp, NeurIPS'23:linearMDP}. Extending their results to the bandit feedback setting is an important future direction.

\section*{Acknowledgements}

This research was supported by NSFC (62206125, U23A20382, 61921006). Peng Zhao was supported in part by the Xiaomi Foundation.

\bibliography{mdp}

\onecolumn
\appendix

\section{Proof of Lemma~\ref{lem:concentration}}

In this section, we present the proof of Lemma~\ref{lem:concentration}, which is a simplified version of the proof of~\citet[Theorem C.6]{NIPS'22:Perivier-dynamic-assortment}. For self-containedness, we present the proof below.

\subsection{Main Proof}

\begin{proof}
  We define a global variable as $z_i = \sum_{j=1}^N \varepsilon_{i,j} x_{i,j}$ and analyze the concentration of $z_i$. Denote $\B_d(x, r)$ the $d$-dimensional $\ell_2$ ball centered at $x$ with radius $r$. For all $\xi \in \B_d(0, 1/2)$, we define
  \begin{align}
    M_0(\xi) = 1, \mbox{ and } M_t(\xi) = \exp(\xi^\top z_t - \norm{\xi}_{Y_t}^2). \label{eq:M}
  \end{align}
  We denote by $\F_t$ the $\sigma$-algebra generated by $\{\{x_i, \varepsilon_i\}_{i=1}^{t-1}, x_t\}$. To prove Lemma~\ref{lem:concentration}, the crucial step is to demonstrate that though $z_i$ is a combination of non-independent variables, relatively to $\F_t$, $M_t(\xi)$ is still a super martingale, which is present in the following lemma.
  \begin{myLemma}
    \label{lem:martingle}
    For all $\xi \in \B_d(0, 1/2)$, $\{M_t(\xi)\}_{t=0}^\infty$ defined in~\eqref{eq:M} is a non-negative super martingale.
  \end{myLemma}
  We present the proof of Lemma~\ref{lem:martingle} in Appendix~\ref{appendix:martingle-proof}. Then, the remaining proof follows the proof of~\citet[Theorem 1]{ICML'20:Faury-improved-logistic}. The main difference in our analysis is that $\xi$ belongs to $\B_d(0, 1/2)$ instead of $\B_d(0, 1)$ to ensure $\{M_t(\xi)\}_{t=0}^\infty$ is a super martingale.

  Let $h(\xi)$ be a probability density function with support on $\B_d(0, 1/2)$. For $t \geq 0$ let:
  \begin{align*}
    \bar{M}_t \triangleq \int_{\xi} M_t(\xi) d h(\xi)
  \end{align*}
  By Lemma 20.3 of~\citet{book:bandit}, $\bar{M}_t$ is also a non-negative super-martingale, and $\mathbb{E}\left[\bar{M}_0\right]=1$. Let $\tau$ be a stopping time with respect to the filtration $\left\{\F_t\right\}_{t=0}^{\infty}$. We can follow the proof of Lemma 8 in~\citet{NIPS'11:AY-linear-bandits} to justify that $\bar{M}_\tau$ is well-defined (independently of whether $\tau<\infty$ holds or not) and that $\mathbb{E}\left[\bar{M}_\tau\right] \leq 1$. Therefore, with $\zeta \in(0,1)$ and thanks to the maximal inequality:
  \begin{align*}
    \Pr \left[\log \left(\bar{M}_\tau\right) \geq \log \left(\frac{1}{\zeta}\right)\right]=\Pr \left[\bar{M}_\tau \geq \frac{1}{\zeta}\right] \leq \zeta.
  \end{align*}
  Then, we compute the lower bound of $\bar{M}_t$ as follows. Let $\beta_t = \sqrt{2 \lambda_t}$ be a positive scalar and set $h$ to be the density function of an isotropic normal distribution of precision $\beta_t^2$ truncated on $\B_d(0, 1/2)$. Denote $N(h)$ its normalization constant. Then, we have
  \begin{align*}
    \bar{M}_t=\frac{1}{N(h)} \int_{\B_d(0, 1/2)} \exp \left(\xi^T S_t-\|\xi\|_{Y_t}^2\right) d \xi.
  \end{align*}
  To simplify the notation, let $f(\xi):=\xi^T S_t-\|\xi\|_{Y_t}^2$ and $\xi_*=\arg \max _{\|\xi\|_2 \leq 1 / 4} f(\xi)$, we obtain
  \begin{align*}
    \bar{M}_t & =\frac{e^{f\left(\xi_*\right)}}{N(h)} \int_{\mathbb{R}^d} \mathbf{1}_{\|\xi\|_2 \leq 1/2} \exp \left(\left(\xi-\xi_*\right)^T \nabla f\left(\xi_*\right)-\left(\xi-\xi_*\right)^T Y_t\left(\xi-\xi_*\right)\right) d \xi \\
              & =\frac{e^{f\left(\xi_*\right)}}{N(h)} \int_{\mathbb{R}^d} \mathbf{1}_{\left\|\xi+\xi_*\right\|_2 \leq 1/2} \exp \left(\xi^T \nabla f\left(\xi_*\right)-\xi^T Y_t \xi\right) d \xi                                        \\
              & \geq \frac{e^{f\left(\xi_*\right)}}{N(h)} \int_{\mathbb{R}^d} \mathbf{1}_{\|\xi\|_2 \leq 1 / 4} \exp \left(\xi^T \nabla f\left(\xi_*\right)-\xi^T Y_t \xi\right) d \xi                                                   \\
              & =\frac{e^{f\left(\xi_*\right)}}{N(h)} \int_{\mathbb{R}^d} \mathbf{1}_{\|\xi\|_2 \leq 1 / 4} \exp \left(\xi^T \nabla f\left(\xi_*\right)\right) \exp \left(-\frac{1}{2} \xi^T\left(2 Y_t\right) \xi\right) d \xi.
  \end{align*}

  Further, we define $g(\xi)$ as the density of the normal distribution of precision $2Y_t$ truncated on the ball $\B_d(0, 1/4)$ and $N(g)$ its normalization constant. Then, we have
  \begin{align*}
    \bar{M}_t \geq \exp \left(f\left(\xi_*\right)\right) \frac{N(g)}{N(h)} \mathbb{E}_g\left[\exp \left(\xi^T \nabla f\left(\xi_*\right)\right)\right] \geq \exp \left(f\left(\xi_*\right)\right) \frac{N(g)}{N(h)} \exp \left(\mathbb{E}_g\left[\xi^T \nabla f\left(\xi_*\right)\right]\right) \geq \exp \left(f\left(\xi_*\right)\right) \frac{N(g)}{N(h)},
  \end{align*}
  where the last inequality holds by $\E_g[\xi] = 0$. Then, we obtain that for all $\xi_0$ such that $\norm{\xi_0}_2 \leq 1/4$:
  \begin{align*}
    \Pr \left[\bar{M}_t \geq \frac{1}{\zeta}\right] & \geq \Pr \left[\exp \left(f\left(\xi_*\right)\right) \frac{N(g)}{N(h)} \geq 1 / \zeta\right]                                       \\
                                                    & =\Pr \left[\log \left(\exp \left(f\left(\xi_*\right)\right) \frac{N(g)}{N(h)}\right) \geq \log (1 / \zeta)\right]                  \\
                                                    & =\Pr \left[f\left(\xi_*\right) \geq \log (1 / \zeta)+\log \left(\frac{N(h)}{N(g)}\right)\right]                                    \\
                                                    & =\Pr \left[\max _{\|\xi\|_2 \leq 1 / 4} \xi^T S_t-\|\xi\|_{Y_t}^2 \geq \log (1 / \zeta)+\log \left(\frac{N(h)}{N(g)}\right)\right] \\
                                                    & \geq \Pr \left[\xi_0^T S_t-\left\|\xi_0\right\|_{Y_t}^2 \geq \log (1 / \zeta)+\log \left(\frac{N(h)}{N(g)}\right)\right].
  \end{align*}
  In particular, we set $ \xi_0 = \frac{Y_t^{-1} S_t}{\left\|S_t\right\|_{Y_t^{-1}}} \frac{\beta_t}{4 \sqrt{2}} \leq \frac{1}{4}$. Thus, we have
  \begin{align}
    \label{eq:concentration-prob}
    \Pr \left[\left\|S_t\right\|_{Y_t^{-1}} \geq \frac{\beta_t}{4 \sqrt{2}}+\frac{4 \sqrt{2}}{\beta_t} \log \left(\frac{N(h)}{\zeta N(g)}\right)\right] \leq \Pr \left[\bar{M}_t \geq \frac{1}{\zeta}\right].
  \end{align}
  It remains to bound the quantities $N(h)$ and $g(h)$. To this end, we introduce the following lemma in~\citet{ICML'20:Faury-improved-logistic} which bounds the log of their ratio.
  \begin{myLemma}[Lemma 6 of~\citet{ICML'20:Faury-improved-logistic}]
    The following inequality holds:
    \begin{align}
      \label{eq:concentration-quantity}
      \log \left(\frac{N(h)}{N(g)}\right) \leq \log \left(\frac{2^{d / 2} \operatorname{det}\left(Y_t\right)^{1 / 2}}{\beta_t^d}\right)+d \log (2).
    \end{align}
  \end{myLemma}
  Combining~\eqref{eq:concentration-prob} and~\eqref{eq:concentration-quantity}, with probability at least $1-\zeta$, for all $t$ it holds that
  \begin{align*}
    \left\|S_t\right\|_{Y_t^{-1}} \leq \frac{\beta_t}{4 \sqrt{2}}+\frac{4 \sqrt{2}}{\beta_t} \log \left(\frac{2^{d / 2} \operatorname{det}\left(Y_t\right)^{1 / 2}}{\beta_t^d \zeta}\right)+\frac{4 \sqrt{2}}{\beta_t} d \log (2).
  \end{align*}
  Finally, the poof is finished by the definition $\beta_t = \sqrt{2 \lambda_t}$.
\end{proof}

\subsection{Proof of Lemma~\ref{lem:martingle}}
\label{appendix:martingle-proof}

\begin{proof}
  Since $\delta_t$ is a one-hot vector, for all $j \in [N]$, there is a single index $j \in [0]$ for which $\delta_j = 1$ and $\delta_{j'} = 0$ for all $j' \neq j$. Besides, we have $\Pr(\delta_{i, j} =1 | \F_i) = p_{i,j}$. Hence, conditional on $\F_i$, the variance of $\xi^\top z_i$ can be written as follows. For simplicity, we denote $\E[\cdot] = \E[\cdot | \F_i]$ below.
  \begin{align*}
         & \sigma^2(\xi^\top z_i | \F_i)                                                                                                                                                                                                                                                         \\
    = \  & \E\Bigg[\bigg(\sum_{j=1}^N(\delta_{i,j}-p_{i,j}) \xi^\top x_{i,j}\bigg)^2\Bigg] - \E\Bigg[\bigg(\sum_{j=1}^N(\delta_{i,j}-p_{i,j}) \xi^\top x_{i,j}\bigg)\Bigg]^2                                                                                                                     \\
    = \  & \E\Bigg[\bigg(\sum_{j=1}^N(\delta_{i,j}-p_{i,j}) \xi^\top x_{i,j}\bigg)^2\Bigg]                                                                                                                                                                                                       \\
    = \  & \E\Bigg[\bigg(\sum_{j=1}^N \sum_{k=1}^N (\delta_{i,j}\xi^\top x_{i,j}) (\delta_{i,k}\xi^\top x_{i,k})\bigg)^2\Bigg] - 2 \E\Bigg[\sum_{j=1}^N \delta_{i,j}\xi^\top x_{i,j}\Bigg] \bigg(\sum_{j=1}^N p_{i,j}\xi^\top x_{i,j}\bigg) + \bigg(\sum_{j=1}^N p_{i,j}\xi^\top x_{i,j}\bigg)^2 \\
    = \  & \E\Bigg[\sum_{j=1}^N \delta_{i,j} (\xi^\top x_{i,j})^2\Bigg] - 2 \bigg(\sum_{j=1}^N p_{i,j} \xi^\top x_{i,j}\bigg)^2 + \bigg(\sum_{j=1}^N p_{i,j}\xi^\top x_{i,j}\bigg)^2                                                                                                             \\
    = \  & \sum_{j=1}^N p_{i,j} (\xi^\top x_{i,j})^2 - \bigg(\sum_{j=1}^N p_{i,j} \xi^\top x_{i,j}\bigg)^2 \leq \sum_{j=1}^N (\xi^\top x_{i,j})^2 = \norm{\xi}_{Y_t}^2 - \norm{\xi}_{Y_{t-1}}^2.
  \end{align*}

  Then, note that $S_{t-1}$ is $\F_t$-measurable, thus for all $t \geq 1$, we have
  \begin{align*}
    \E \left[\exp(\xi^\top S_t) | \F_t\right] = \exp(\xi^\top S_{t-1}) \E \left[\exp(\xi^\top z_t) | \F_t\right].
  \end{align*}
  Next we apply Lemma~\ref{lem:variance-bound} to bound the term $\E \left[\exp(\xi^\top z_t) | \F_t\right]$ and we need to ensure $\abs{\xi^\top z_t} \leq 1$. Since $\delta_{t}$ is an one-hot vector, let $j$ be the index such that $\delta_{t, j} = 1$. Then, we have $\delta_{t, k} = 0$ for all $k \in [N] \backslash \{j\}$. Note $\norm{x_{t,j}} \leq 1$ for all $t, j$ and $\norm{\xi} \leq 1/2$, thus we have
  \begin{align*}
    \abs{\xi^\top z_{t}} \leq (1 - p_{t, j}) \abs{\xi^\top x_{t, j}} + \sum_{k \in [N] \backslash \{j\}} p_{t, k} \abs{\xi^\top x_{t, k}} \leq \frac{1}{2} \left(1 + \sum_{j \in [N]} p_{t, j}\right) \leq 1.
  \end{align*}
  Since $\abs{\xi^\top z_t} \leq 1$, we can apply Lemma~\ref{lem:variance-bound} and obtain
  \begin{align}
    \E \left[\exp(\xi^\top S_t) | \F_t\right] & = \exp(\xi^\top S_{t-1}) \E \left[\exp(\xi^\top z_t) | \F_t\right] \leq \exp(\xi^\top S_{t-1}) (1 + \sigma^2(\xi^\top z_t | \F_t)) \leq \exp(\xi^\top S_{t-1} + \sigma^2(\xi^\top z_t | \F_t)), \label{eq:expectation-bound}
  \end{align}
  where the first inequality holds by Lemma~\ref{lem:variance-bound} and the last inequality holds by $ 1 + x \leq e^x$.

  Finally, we obtain
  \begin{align*}
    \E \left[M_t(\xi) | \F_t\right]                                               = \E\left[\exp(\xi^\top S_{t} - \norm{\xi}_{Y_t}) | \F_t\right]      = \  & \E\left[\exp(\xi^\top S_{t}) | \F_t\right] \exp(- \norm{\xi}_{Y_t})           \\
    \leq \                                                                                                                                                  & \exp(\xi^\top S_{t-1} + \sigma^2(\xi^\top z_t | \F_t)^2 - \norm{\xi}_{Y_t}^2) \\
    \leq \                                                                                                                                                  & \exp(\xi^\top S_{t-1} + \norm{\xi}_{Y_{t-1}}^2)                               \\
    = \                                                                                                                                                     & M_{t-1}(\xi),
  \end{align*}
  where the first equality holds since $Y_t$ is $\F_t$-measurable, the first inequality holds by~\eqref{eq:expectation-bound} and the second inequality holds by $\sigma^2(\xi^\top z_t | \F_t)^2 \leq \norm{\xi}_{Y_t}^2 - \norm{\xi}_{Y_{t-1}}^2$. This shows that $\{M_t(\xi)\}_{t=0}^\infty$ is a super martingale.
\end{proof}

\section{Proof of Lemma~\ref{lem:confidence-set}}

\begin{proof}
  Recall the closed-form of $\theta_{k,h}$ is given by
  \begin{align*}
    \theta_{k,h} = \Lambda_{k,h}^{-1} \sum_{i=1}^{k-1} \sum_{s' \in \S_{h+1}} \delta_{s_{i, h+1}}(s') \phi(s' \given s_{i,h}, a_{i,h}).
  \end{align*}
  where $\Lambda_{k,h} = \sum_{i=1}^{k-1}\sum_{s' \in \S_{h+1}} \phi(s' | s_{i,h}, a_{i,h})\phi(s' | s_{i,h}, a_{i,h})^\top + \lambda_k I_d$. We decompose the closed form as follows.
  \begin{align*}
    \theta_{k,h} & = \Lambda_{k,h}^{-1} \sum_{i=1}^{k-1} \sum_{s' \in \S_{h+1}} \delta_{s_{i, h+1}}(s') \phi(s' \given s_{i,h}, a_{i,h})                                                                                                                                                \\
                 & = \Lambda_{k,h}^{-1} \sum_{i=1}^{k-1} \sum_{s' \in \S_{h+1}} (P_h(s' \given s_{i,h}, a_{i,h}) + \varepsilon_{i,h}(s')) \phi(s' \given s_{i,h}, a_{i,h})                                                                                                              \\
                 & = \Lambda_{k,h}^{-1} \sum_{i=1}^{k-1} \sum_{s' \in \S_{h+1}} P_h(s' \given s_{i,h}, a_{i,h}) \phi(s' \given s_{i,h}, a_{i,h}) + \Lambda_{k,h}^{-1} \sum_{i=1}^{k-1} \sum_{s' \in \S_{h+1}} \varepsilon_{i,h}(s') \phi(s' \given s_{i,h}, a_{i,h})                    \\
                 & = \Lambda_{k,h}^{-1} \sum_{i=1}^{k-1} \sum_{s' \in \S_{h+1}} \phi_h(s' \given s_{i,h}, a_{i,h})^\top \theta_h^* \phi(s' \given s_{i,h}, a_{i,h}) + \Lambda_{k,h}^{-1} \sum_{i=1}^{k-1} \sum_{s' \in \S_{h+1}} \varepsilon_{i,h}(s') \phi(s' \given s_{i,h}, a_{i,h}) \\
                 & = \Lambda_{k,h}^{-1} (\Lambda_{k,h} - \lambda_k I_d)\theta_h^* + \Lambda_{k,h}^{-1} \sum_{i=1}^{k-1} \sum_{s' \in \S_{h+1}} \varepsilon_{i,h}(s') \phi(s' \given s_{i,h}, a_{i,h})                                                                                   \\
                 & = \theta_h^* + \lambda_k \Lambda_{k,h}^{-1} \theta_h^* + \Lambda_{k,h}^{-1} \sum_{i=1}^{k-1} \sum_{s' \in \S_{h+1}} \varepsilon_{i,h}(s') \phi(s' \given s_{i,h}, a_{i,h}).
  \end{align*}
  Rearranging terms, we obtain for all $\zeta \in (0,1)$, with probability at least $1 - \zeta/H$, it holds that
  \begin{align*}
    \norm{\theta_{k,h} - \theta_h^*}_{\Lambda_{k,h}}
    =    & \ \lambda_k \norm{ \theta_h^*}_{\Lambda_{k,h}} + \Bignorm{\sum_{i=1}^{k-1} \sum_{s' \in \S_{h+1}} \varepsilon_{i,h}(s') \phi(s' \given s_{i,h}, a_{i,h})}_{\Lambda_{k,h}}                \\
    \leq & \ \sqrt{\lambda_k} B + \frac{\sqrt{\lambda_k}}{4}+\frac{4}{\sqrt{\lambda_k}} \log \left(\frac{2^d \det\left(\Lambda_{k,h}\right)^{\frac{1}{2}} \lambda_k^{-\frac{d}{2}}}{\zeta/H}\right) \\
    \leq & \ \sqrt{\lambda_k}(B + \frac{1}{4}) + \frac{4}{\sqrt{\lambda_k}}\left(\log \Big(\frac{H}{\zeta}\Big) + \frac{d}{2} \log\Big(4+\frac{4Sk}{\lambda_k d}\Big)\right),
  \end{align*}
  where the first inequality holds by the self-normalized concentration of Lemma~\ref{lem:concentration}, and the last inequality holds by the determinant-trace inequality in Lemma~\ref{lem:det-trace-inequality}. This shows with probability at least $1 - \zeta/H$, it holds that $\theta_h^* \in \C_{k,h}$. Applying a union bound over $h = 1, \ldots, H$ finishes the proof.
\end{proof}

\section{Proof of Lemma~\ref{lem:occupancy-measure-difference}}

In this section, we present the proof of Lemma~\ref{lem:occupancy-measure-difference}.

\subsection{Main Proof}
\begin{proof}
  By Lemma~\ref{lem:transition-difference}, for any $(s, a) \in \S_m \times \A, m \in [H]$,  we have
  \begin{align*}
    \sum_{s' \in \S_{m+1}} \epsilon_{k,m}(s' \given s, a) = \sum_{s' \in \S_{m+1}} |\Ph_{k,m}(s' | s, a) - P_m(s' | s, a)| \leq 2 \wedge \sum_{s' \in \S_{m+1}} \beta_{k}\norm{\phi(s' \given s, a)}_{\Lambda_{k,m}^{-1}},
  \end{align*}
  where the last inequality holds by the fact $\norm{\Ph_{k,m}(\cdot \given s, a) - P_m(\cdot \given s, a)}_1 \leq \norm{\Ph_{k,m}(\cdot \given s, a)}_1 + \norm{P_m(\cdot \given s, a)}_1 = 2$.

  Then, with probability at least $1 - \zeta$, it holds that
  \begin{align*}
     & \quad \sum_{k=1}^K \sum_{(s, a) \in \mathcal{S} \times \mathcal{A}}\left|q_k^s(s, a)-q_k(s, a)\right|                                                                                                \\
     & \leq 2 S \sumk \sum_{m=1}^H \sum_{s' \in \S_{m+1}} \epsilon_{k,h}(s' \mid s_{k,m}, a_{k,m}) + 4 S^2 \log \Big(\frac{H}{\zeta}\Big)                                                                   \\
     & \leq 4 S \sumk \sum_{m=1}^H  \beta_{k}(1 \wedge \sum_{s' \in \S_{m+1}}\norm{\phi(s' \given s_{k,m}, a_{k,m})}_{\Lambda_{k,m}^{-1}}) + 4 S^2 \log\Big(\frac{H}{\zeta}\Big)                            \\
     & \leq 4\beta_K S \sum_{m=1}^H \sumk  (1 \wedge \sum_{s' \in \S_{m+1}}\norm{\phi(s' \given s_{k,m}, a_{k,m})}_{\Lambda_{k,m}^{-1}}) + 4 S^2 \log\Big(\frac{H}{\zeta}\Big)                              \\
     & \leq 2\beta_K S \sum_{m=1}^H \sqrt{KS_{m+1} \sumk  \left(1 \wedge \sum_{s' \in \S_{m+1}}\norm{\phi(s' \given s_{k,m}, a_{k,m})}_{\Lambda_{k,m}^{-1}}^2\right)} + 4 S^2 \log\Big(\frac{H}{\zeta}\Big) \\
     & \leq 2 \beta_K S \sum_{m=1}^H \sqrt{KS_{m+1} d\log \left(\lambda_{K+1} + \frac{ KS_{m+1}}{d}\right)} + 4 S^2 \log\Big(\frac{H}{\zeta}\Big)                                                           \\
     & \leq 2 \beta_K S \sqrt{KH \sum_{m=1}^H S_{m+1} d\log \left(\lambda_{K+1} + \frac{ KS}{d}\right)} + 4 S^2 \log\Big(\frac{H}{\zeta}\Big)                                                               \\
     & \leq \O\left((d\sqrt{HS^3K} + S^2)\log^2\left(\frac{dKS}{\zeta}\right)\right)                                                                                                                        \\
     & \leq \O\left(d\sqrt{HS^3K}\log^2\left(\frac{dKS}{\zeta}\right)\right),
  \end{align*}
  where the first inequality follows from Lemma~\ref{lem:occupancy-measure-difference-intermediate}, the fourth and sixth inequality holds by the Cauchy-Schwarz inequality, the fifth inequality holds by the specifically designed elliptical potential lemma in Lemma~\ref{lem:generalized-elliptical-potential}, the second last inequality holds by $\lambda_{K+1} = d \log((K+1)S)$ and $\beta_K = \O(\sqrt{d \log(KS)})$ and the last bound holds by $S \leq K$ (otherwise the bound $\sqrt{HSAK}$ becomes vacuous). This completes the proof.
\end{proof}

\subsection{Proof of Lemma~\ref{lem:transition-difference}}

\begin{proof}
  This lemma was first proved in~\citet[Lemma 3]{ICLR'23:bandit-unknown-mixture}. We present their proof for self-containedness.

  By the definition of linear mixture MDPs, for all $k \in [K]$, $h \in [H]$ and $\forall (s, a, s') \in \S_h \times \A \times \S_{h+1}$, we have
  \begin{align*}
    \Bigabs{\Ph_{k,h} (s' \given s, a) - P_h(s' \given s, a)} & = \Bigabs{\phi(s' \given s, a)^\top (\thetah_{k,h} - \theta_h^*)}                                       \\
                                                              & \leq \norm{\phi(s' \given s, a)}_{\Lambda_{k,h}^{-1}} \norm{\thetah_{k,h} - \theta_h^*}_{\Lambda_{k,h}} \\
                                                              & \leq \beta_{k} \norm{\phi(s' \given s, a)}_{\Lambda_{k,h}^{-1}}                                         \\
                                                              & \leq  1 \wedge \beta_{k} \norm{\phi(s' \given s, a)}_{\Lambda_{k,h}^{-1}},
  \end{align*}
  where the first inequality follows from the Holder's inequality, the second inequality holds by Lemma~\ref{lem:confidence-set}, and the last inequality follows from the fact that $\abs{\Ph_{k,h} (s' \given s, a) - P_h(s' \given s, a)} \leq 1$. This completes the proof.
\end{proof}

\subsection{Proof of Lemma~\ref{lem:occupancy-measure-difference-intermediate}}

\begin{proof}
  The main proof is similar to the proof of~\citet[Lemma 2]{ICLR'23:bandit-unknown-mixture}. Let $q_k^s = q^{P_k^s, \pi_k}$ for simplicity, and define $q(s) = \sum_{a \in \A} q(s, a)$. For any $q$ and any $(s,a)$, we have
  \begin{align*}
         & q(s, a)  = q(s) \pi^q(a \given s)                                                                                                                                                                                                             \\
    = \  & \pi^q(s, a) \sum_{s' \in \S_{h(s)-1}} q(s') \sum_{a' \in \A} \pi^q(a' \given s') P^q(s \given s', a')                                                                                                                                         \\
    = \  & \pi^q(s \given a) \sum_{\left\{s_i, a_i\right\}_{i=1}^{h(s)-1} \in \prod_{i=1}^{h(s)-1} \mathcal{S}_i \times \mathcal{A}} \prod_{h=1}^{h(s)-1} \pi^q\left(a_h \given s_h\right) \prod_{h=1}^{h(s)-1} P^q\left(s_{h+1} \given s_h, a_h\right),
  \end{align*}
  where the last equality holds by expressing $q(s_{i+1})$ using $q(s_i)$ recursively for $i = h(s)-1, \ldots, 1$. In the following, we drop the superscript $\prod_{i=1}^{h(s)-1} \mathcal{S}_i \times \mathcal{A}$ of $\left\{s_i, a_i\right\}_{i=1}^{h(s)-1}$ for simplicity. Then, we have
  \begin{align*}
         & \left|q_k^s(s, a)-q_k(s, a)\right|                                                                                                                                                                                                                          \\
    = \  & \pi_k(s \given a) \sum_{\left\{s_i, a_i\right\}_{i=1}^{h(s)-1}} \prod_{h=1}^{h(s)-1} \pi_k\left(a_h \given s_h\right)\left(\prod_{h=1}^{h(s)-1} P_k^s\left(s_{h+1} \given s_h, a_h\right)-\prod_{h=1}^{h(s)-1} P\left(s_{h+1} \given s_h, a_h\right)\right)
  \end{align*}
  Further, we decompose the difference of the transition as follows.
  \begin{align*}
         & \prod_{h=1}^{h(s)-1} P_k^s\left(s_{h+1} \given s_h, a_h\right)-\prod_{h=1}^{h(s)-1} P\left(s_{h+1} \given s_h, a_h\right)                                                                                                                                                \\
    =    & \prod_{h=1}^{h(s)-1} P_k^s\left(s_{h+1} \given s_h, a_h\right)-\prod_{h=1}^{h(s)-1} P\left(s_{h+1} \given s_h, a_h\right) \pm \sum_{m=1}^{h(s)-1} \prod_{h=1}^{m-1} P\left(s_{h+1} \given s_h, a_h\right) \prod_{h=m}^{h(s)-1} P_k^s\left(s_{h+1} \given s_h, a_h\right) \\
    =    & \sum_{m=1}^{h(s)-1}\left(P_k^s\left(s_{m+1} \given s_m, a_m\right)-P\left(s_{m+1} \given s_m, a_m\right)\right) \prod_{h=1}^{m-1} P\left(s_{h+1} \given s_h, a_h\right) \prod_{h=m+1}^{h(s)-1} P_k^s\left(s_{h+1} \given s_h, a_h\right)                                 \\
    \leq & \sum_{m=1}^{h(s)-1}\epsilon_{k,m}(s_{m+1} \given s_m, a_m) \prod_{h=1}^{m-1} P\left(s_{h+1} \given s_h, a_h\right) \prod_{h=m+1}^{h(s)-1} P_k^s\left(s_{h+1} \given s_h, a_h\right),
  \end{align*}
  where $\epsilon_{k,h}(s' \given s, a) = 1 \wedge \beta_{k}\norm{\phi(s' \given s, a)}_{\Lambda_{k,h}^{-1}}$. Therefore, we have
  \begin{align*}
         & \left|q_k^s(s, a)-q_s(s, a)\right|                                                                                                                                                                                                                                                                                     \\
    \leq {} & \pi_k(s \given a) \sum_{\left\{s_i, a_i\right\}_{i=1}^{h(s)-1}} \prod_{h=1}^{h(s)-1} \pi_k\left(a_h \given x_h\right) \sum_{m=1}^{h(s)-1} \epsilon_{k, m}\left(x_{m+1} \given x_m, a_m\right) \prod_{h=1}^{m-1} P\left(s_{h+1} \given s_h, a_h\right) \prod_{h=m+1}^{h(s)-1} P_k^s\left(s_{h+1} \given s_h, a_h\right) \\
    =    {} & \sum_{m=1}^{h(s)-1} \sum_{\left\{s_i, a_i\right\}_{i=1}^{h(s)-1}} \epsilon_{k, m}\left(s_{m+1} \given s_m, a_m\right)\left(\pi_k\left(a_m \given s_m\right) \prod_{h=1}^{m-1} \pi_k\left(a_h \given s_h\right) P\left(s_{h+1} \given s_h, a_h\right)\right)                                                            \\
         & \hspace{70mm} \cdot \left(\pi_k(s \given a) \prod_{h=m+1}^{h(s)-1} \pi_k\left(a_h \given x_h\right) P_k^s\left(s_{h+1} \given s_h, a_h\right)\right)                                                                                                                                                                   \\
    =    {} & \sum_{m=1}^{h(s)-1} \sum_{s_m, a_m, s_{m+1}} \epsilon_{k, m}\left(s_{m+1} \given s_m, a_m\right)\left(\sum_{\left\{s_i, a_i\right\}_{i=1}^{m-1}} \pi_k\left(a_m \given x_m\right) \prod_{h=1}^{m-1} \pi_k\left(a_h \given s_h\right) P\left(s_{h+1} \given s_h, a_h\right)\right)                                      \\
         & \hspace{60mm} \cdot \left(\sum_{a_{m+1}} \sum_{\left\{s_i, a_i\right\}_{i=m+2}^{h(s)-1}} \pi_k(s \given a) \prod_{h=m+1}^{h(s)-1} \pi_k\left(a_h \given s_h\right) P_k^s\left(s_{h+1} \given s_h, a_h\right)\right)                                                                                                    \\
    =    {} & \sum_{m=1}^{h(s)-1} \sum_{s_m, a_m, s_{m+1}} \epsilon_{k, m}\left(s_{m+1} \given s_m, a_m\right) q_k\left(s_m, a_m\right) \pi_k(a \given s) q_k^s\left(s \given s_{m+1}\right)                                                                                                                                         \\
    \leq {} & \pi_k(a \given s) \sum_{m=1}^{h(s)-1} \sum_{s_m, a_m, s_{m+1}} \epsilon_{k, m}\left(s_{m+1} \given s_m, a_m\right) q_k\left(s_m, a_m\right),
  \end{align*}
  where the last inequality holds by $q_k^s\left(s \given s_{m+1}\right) \leq 1$.

  Let $w_m = (s_m, a_m, s_{m+1})$ to simplify the notation. Then, summing over $k \in [K]$ and $(s, a) \in \S \times \A$, we have
  \begin{align}
         {} & \sum_{k=1}^K \sum_{(s, a) \in \mathcal{S} \times \mathcal{A}}\left|q_k^s(s, a)-q_k(s, a)\right| \nonumber                                                                                                  \\
    \leq {} & \sum_{k, s, a} \pi_k(a \mid s) \sum_{m=1}^{h(s)-1} \sum_{w_m} \epsilon_{k, m}\left(s_{m+1} \mid s_m, a_m\right) q_k\left(s_m, a_m\right) \nonumber                                                         \\
    =    {} & \sum_k \sum_{h \leq H} \sum_{m=1}^{h-1} \sum_{w_m} \epsilon_{k, m}\left(s_{m+1} \mid s_m, a_m\right) q_k\left(s_m, a_m\right) \sum_{(s, a) \in \mathcal{S}_h \times \mathcal{A}} \pi_k(a \mid s) \nonumber \\
    =    {} & \sum_{1 \leq m<h \leq H} \sum_{k, w_m} \epsilon_{k, m}\left(s_{m+1} \mid s_m, a_m\right) q_k\left(s_m, a_m\right)\left|\S_h\right| \nonumber                                                               \\
    \leq {} & S \sum_{1 \leq m\leq H} \sum_{k, w_m} \epsilon_{k, m}\left(s_{m+1} \mid s_m, a_m\right) q_k\left(s_m, a_m\right) \label{eq:diff-1}
  \end{align}

  Then, we focus on $\sum_{k, w_m} \epsilon_{k, m}\left(s_{m+1} \mid s_m, a_m\right) q_k\left(s_m, a_m\right)$ with a fixed $m$ at first:
  \begin{align}
      & \sum_{k, w_m} \epsilon_{k, m}\left(s_{m+1} \mid s_m, a_m\right) q_k\left(s_m, a_m\right) \nonumber                                                                                                                                                                                                                       \\
    = & \underbrace{\sum_{k, w_m} \indicator_k\{s_m, a_m\} \epsilon_{k, m}\left(s_{m+1} \mid s_m, a_m\right)}_{\termone} + \underbrace{\sum_{k, w_m} S_{m+1} \left(\frac{q_k(s_m, a_m)}{S_{m+1}} - \frac{\indicator_k(s_m, a_m)}{S_{m+1}}\right) \epsilon_{k, m}\left(s_{m+1} \mid s_m, a_m\right)}_{\termtwo} \label{eq:diff-2}
  \end{align}
  For $\termone$, since $\indicator_k(s, a)$ is the indicator whether the pair $(s, a)$ is visited in episode $k$, thus we have
  \begin{align}
    \sum_{k, w_m} \indicator_k\{s_m, a_m\} \epsilon_{k, m}\left(s_{m+1} \mid s_m, a_m\right) = \sumk \sum_{s' \in \S_{m+1}} \epsilon_{k,h}(s' \mid s_{k,m}, a_{k,m}) \label{eq:term-1}
  \end{align}
  To bound $\termtwo$, we first use Lemma~\ref{lem:bernstein-martingale} to build the connection between $\termone$ and $\termtwo$. Let
  \begin{align*}
    Y_{k, m}=\sum_{w_m}\left(\frac{q_k\left(s_m, a_m\right)}{S_{m+1}}-\frac{\mathbb{I}_k\left\{s_m, a_m\right\}}{S_{m+1}}\right) \epsilon_{k, m}\left(s_{m+1} \mid s_m, a_m\right).
  \end{align*}
  It is easy to verify that $Y_{k,m} \leq 1$. Let $o_{i,j} = (s_{i, j}, a_{i,j}, \ell_i(s_{i,j}, a_{i, j}))$ be the observations in episode $i$, we denote $\F_{k, h}$ the $\sigma$-algebra generated by $\{o_{i,j}\}_{i=1, j=1}^{k, h}$. Then, we have
  \begin{align*}
    \mathbb{E}_{k-1, H}\left[Y_{k, m}^2\right] \leq & \frac{\mathbb{E}_{k-1, H}\left[\left(\sum_{w_m} \mathbb{I}_k\left(s_m, a_m\right)\epsilon_{k, m}\left(s_{m+1} \mid s_m, a_m\right)\right)^2\right]}{S_{m+1}^2} \\
    =                                               & \frac{\mathbb{E}_{k-1, H}\left[\sum_{w_m} \mathbb{I}_k\left(s_m, a_m\right)\epsilon_{k, m}\left(s_{m+1} \mid s_m, a_m\right)^2\right]}{S_{m+1}^2}              \\
    \leq                                            & \frac{\sum_{w_m} q_k\left(s_m, a_m\right)\epsilon_{k, m}\left(s_{m+1} \mid s_m, a_m\right)}{S_{m+1}}
  \end{align*}
  where the equality follows from the fact that $\indicator_k(s_m, a_m) \indicator_k(s'_m, a'_m)$ for $s_m \neq s'_m$, and the last inequality holds by $\epsilon_{k, m}\left(s_{m+1} \mid s_m, a_m\right) \leq 1$ and $\epsilon_{k,m}$ is $\F_{k-1, H}$-measurable. Then, by choosing $\lambda = 1/2$ in Lemma~\ref{lem:bernstein-martingale}, with probability at least $1 - \zeta/H$, we have
  \begin{align*}
         & \sumk \sum_{w_m}\left(\frac{q_k\left(s_m, a_m\right)}{S_{m+1}}-\frac{\mathbb{I}_k\left\{s_m, a_m\right\}}{S_{m+1}}\right) \epsilon_{k, m}\left(s_{m+1} \mid s_m, a_m\right) \\
    \leq & \ \frac{1}{2 S_{m+1}} \sumk \sum_{w_m}q_k\left(s_m, a_m\right) \epsilon_{k, m}\left(s_{m+1} \mid s_m, a_m\right) + 2 \log(H/\zeta)
  \end{align*}
  By applying with a union bound over $m = 1, \ldots, H$, we have with probability at least $1 - \zeta$, it holds that
  \begin{align*}
         & \sumk \sum_{w_m}\left(q_k\left(s_m, a_m\right)-\mathbb{I}_k\left\{s_m, a_m\right\}\right) \epsilon_{k, m}\left(s_{m+1} \mid s_m, a_m\right) \\
    \leq & \ \frac{1}{2} \sumk \sum_{w_m}q_k\left(s_m, a_m\right) \epsilon_{k, m}\left(s_{m+1} \mid s_m, a_m\right) + 2 S_{m+1}\log(H/\zeta).
  \end{align*}
  This shows that
  \begin{align}
    \label{eq:term-relation}
    \termtwo \leq \sumk \sum_{w_m} \indicator_k\left\{s_m, a_m\right\} \epsilon_{k, m}\left(s_{m+1} \mid s_m, a_m\right) + 4 S_{m+1}\log(H/\zeta) \leq \termone + 4 S_{m+1}\log(H/\zeta).
  \end{align}
  Combining~\eqref{eq:diff-1} and~\eqref{eq:diff-2}, we have
  \begin{align*}
         & \sum_{k=1}^K \sum_{(s, a) \in \mathcal{S} \times \mathcal{A}}\left|q_k^s(s, a)-q_k(s, a)\right|                 \\                               \leq & \ S \sum_{1 \leq m\leq H} (\termone + \termtwo)                                                                                   \\
    \leq & \ S \sum_{1 \leq m\leq H} (2 \termone + 4 S_{m+1}\log(H/\zeta))                                                 \\
    \leq & \ 2 S \sumk \sum_{m=1}^H \sum_{s' \in \S_{m+1}} \epsilon_{k,h}(s' \mid s_{k,m}, a_{k,m}) + 4 S^2 \log(H/\zeta),
  \end{align*}
  where the second inequality holds by~\eqref{eq:term-relation} and the last inequality holds by~\eqref{eq:term-1}. This completes the proof.
\end{proof}

\section{Proof of Theorem~\ref{thm:regret-bound}}
\label{appendix:proof-regret-bound}

In this section, we present the proof of Theorem~\ref{thm:regret-bound}.

\subsection{Main Proof}
\begin{proof}
  Define the occupancy measure under the true transition $P$ and policy $\pi_k$ as $q_k = q^{P, \pi_k}$. Then, the regret can be written as $\Reg = \sumk \inner{q_k - q^*}{\ell_k}$. As in Section~\ref{sec:proof-sketch}, we decompose the regret as the follows:
  \begin{align*}
    \Reg(K) \leq & \underbrace{\sumk \inner{\qh_k - q^*}{\ellh_k}}_{\regret} + \underbrace{\sumk \inner{q_k - \qh_k}{\ell_k}}_{\error} + \underbrace{\sumk \inner{\qh_k}{\ell_k - \ellh_k}}_{\biasone} + \underbrace{\sumk \inner{q^*}{\ellh_k - \ell_k}}_{\biastwo}.
  \end{align*}
  The bounds of $\error$ and $\biasone$ term are shown in Lemma~\ref{lem:error} and Lemma~\ref{lem:bias1} of Section~\ref{sec:regret}, respectively. We bound the $\regret$ and $\biastwo$ terms below.

  \textbf{Bounding} $\biastwo$ \textbf{Term.} For this term, we present the following lemma, whose proof is in Appendix~\ref{appendix:proof-bias2}.

  \begin{myLemma}
    \label{lem:bias2}
    For any $\zeta \in (0, 1)$, with probability at least $1 - 2\zeta$, \textsf{VLSUOB-REPS} algorithm ensures that
    \begin{align*}
      \textnormal{\biastwo} \leq \O\left(\frac{H \log(SA/\zeta)}{\gamma}\right).
    \end{align*}
  \end{myLemma}

  \textbf{Bounding} $\regret$ \textbf{Term.} For this term, we present the following lemma, whose proof is in Appendix~\ref{appendix:proof-regret}.

  \begin{myLemma}
    \label{lem:regret}
    For any $\zeta \in (0, 1)$, with probability at least $1 - 2\zeta$, \textsf{VLSUOB-REPS} algorithm ensures that
    \begin{align*}
      \textnormal{\regret} \leq \O\left(\frac{H \log(SA/\zeta)}{\eta} + \eta SAK + \frac{\eta H \log(H/\zeta)}{\gamma}\right).
    \end{align*}
  \end{myLemma}

  Combining Lemma~\ref{lem:error}, Lemma~\ref{lem:bias1}, Lemma~\ref{lem:bias2} and Lemma~\ref{lem:regret} finishes the proof.
\end{proof}

\subsection{Proof of Lemma~\ref{lem:bias2}}
\label{appendix:proof-bias2}

\begin{proof}
  For some $(s, a) \in \S \times \A$, using $\alpha_k(s', a') = 2\gamma \indicator\{(s', a') = (s, a)\}$, we have with probability at least $1 - \frac{\zeta}{SA}$,
  \begin{align*}
    \sum_{k=1}^K\left(\widehat{\ell}_k(s, a)-\frac{q_k(s, a)}{u_k(s, a)} \ell_k(s, a)\right) \leq \frac{1}{2 \gamma} \log \left(\frac{S A}{\zeta}\right)
  \end{align*}
  By using a union bound, the above inequality holds for all $(s, a) \in \S \times \A$ simultaneously with probability at least $1 - \zeta$. Further, under the event that $\theta_h^* \in \C_{k,h}$, we have $q_k(s, a) \leq u_k(s, a)$, which implies that
  \begin{align*}
    \sum_{k=1}^K\left\langle q^*, \widehat{\ell}_k-\ell_k\right\rangle & \leq \sum_{k, s, a} q^*(s, a) \ell_k(s, a)\left(\frac{q_k(s, a)}{u_k(s, a)}-1\right)+\sum_{s, a} \frac{q^*(s, a) \log \frac{S A}{\zeta}}{2 \gamma} \\
                                                                       & =\sum_{k, s, a} q^*(s, a) \ell_k(s, a)\left(\frac{q_k(s, a)}{u_k(s, a)}-1\right)+\frac{H \log \frac{S A}{\zeta}}{2 \gamma}                         \\
                                                                       & \leq \frac{H \log \frac{S A}{\zeta}}{2 \gamma} .
  \end{align*}
  The proof is concluded by applying the union bound again.
\end{proof}

\subsection{Proof of Lemma~\ref{lem:regret}}
\label{appendix:proof-regret}

\begin{proof}
  The update procedure in~\eqref{eq:omd} can be written as the following two-step procedure.
  \begin{align*}
     & \widetilde{q}_{k+1}=\argmin_{q \in \mathbb{R}_{+}^{S A}} \eta\left\langle q, \widehat{\ell}_k\right\rangle+\Dp\left(q, \widehat{q}_k\right), \\
     & \widehat{q}_{k+1}=\argmin_{q \in \Delta\left(\mathcal{P}_{k+1}\right)}\Dp\left(q, \widetilde{q}_{k+1}\right),
  \end{align*}
  The closed form of $\widetilde{q}_{k+1}$ is given by $\qt_{k+1}(s,a) = \qh_{k+1}(s, a)\exp(-\eta \ellh_k(s, a))$. Then, we have
  \begin{align*}
    \left\langle\widehat{q}_k-q^*, \widehat{\ell}_k\right\rangle & =\frac{1}{\eta}\left(\Dp\left(q^*, \widehat{q}_k\right)+\Dp\left(\widehat{q}_k, \widetilde{q}_{k+1}\right)-\Dp\left(q^*, \widetilde{q}_{k+1}\right)\right)    \\
                                                                 & \leq \frac{1}{\eta}\left(\Dp\left(q^*, \widehat{q}_k\right)+\Dp\left(\widehat{q}_k, \widetilde{q}_{k+1}\right)-\Dp\left(q^*, \widehat{q}_{k+1}\right)\right),
  \end{align*}
  where the equality holds by the three-point equality, and the inequality holds by the generalized Pythagorean theorem. Then, summing over $k \in [K]$ and using the telescoping argument, we have
  \begin{align*}
    \sum_{k=1}^K\left\langle\widehat{q}_k-q^*, \widehat{\ell}_k\right\rangle \leq \frac{1}{\eta}\left(\Dp\left(q^*, \widehat{q}_1\right)-\Dp\left(q^*, \widehat{q}_{K+1}\right)+\sum_{k=1}^K \Dp\left(\widehat{q}_k, \widetilde{q}_{k+1}\right)\right).
  \end{align*}
  The first two terms can be rewritten as
  \begin{align*}
    \Dp\left(q^*, \widehat{q}_1\right)-\Dp\left(q^*, \widehat{q}_{K+1}\right) = \sum_{h=1}^{H} \sum_{s \in \mathcal{S}_h} \sum_{a \in \mathcal{A}} q^*(s, a) \log \frac{\widehat{q}_{K+1}(s, a)}{\widehat{q}_1(s, a)} \leq \sum_{h=1}^{H} \sum_{s \in \mathcal{S}_h} \sum_{a \in \mathcal{A}} q^*(s, a) \log \left(S_h A\right) \leq H \log (S A).
  \end{align*}
  It remains to bound the last term.
  \begin{align*}
    \Dp\left(\widehat{q}_k, \widetilde{q}_{k+1}\right) & =\sum_{h=1}^{H} \sum_{s \in \S_h} \sum_{a \in \A}\left(\eta \widehat{q}_k\left(s, a\right) \widehat{\ell}_k(s, a)-\widehat{q}_k\left(s, a\right)+\widehat{q}_k\left(s, a\right) \exp \left(-\eta \widehat{\ell}_k(s, a)\right)\right)                                                                                  \\
                                                       & \leq \eta^2 \sum_{h=1}^{H} \sum_{s \in \S_h} \sum_{a \in \A} \widehat{q}_k\left(s, a\right) \widehat{\ell}_k(s, a)^2                                                                                                                   =\eta^2 \sum_{s \in \S, a \in \A} \widehat{q}_k(s, a) \widehat{\ell}_k(s, a)^2,
  \end{align*}
  where the inequality is due to the fact that $e^{-x} \leq 1 - x + x^2$ for all $x \geq 0$.

  Note that due to the definition of $\ellh_k(s, a)$, we have
  \begin{align*}
    \widehat{q}_k(s, a) \widehat{\ell}_k(s, a)^2=\frac{\widehat{q}_k(s, a) \ell_k(s, a) \mathbb{I}_k\{s, a\}}{u_k(s, a)+\gamma} \widehat{\ell}_k(s, a) \leq \widehat{\ell}_k(s, a),
  \end{align*}
  which is due to the fact that $\qh_k(s, a) \leq u_k(s, a)$ and $\ell_k(s, a) \mathbb{I}_k\{s, a\} \leq 1$. Furthermore, using Lemma~\ref{lem:bernstein-martingale} by setting $\alpha_k(s, a) = 2 \gamma$, with probability at least $1 - \zeta$, we have
  \begin{align*}
    \sum_{k, s, a} \widehat{q}_k(s, a) \widehat{\ell}_k(s, a)^2 \leq \sum_{k, s, a} \frac{q_k(s, a)}{u_k(s, a)} \ell_k(s, a)+\frac{H \log (\frac{H}{\zeta})}{2 \gamma} \leq S A K+\frac{H \log (\frac{H}{\zeta})}{2 \gamma}
  \end{align*}
  where the last inequality comes from that $q_k(s, a) \leq u_k(s, a)$ and $\ell_k(s, a) \leq 1$.

  Applying a union bound over the above bounds, with probability at least $1 - 2\zeta$, we have
  \begin{align*}
    \sum_{k=1}^K\left\langle\widehat{q}_k-q^*, \widehat{\ell}_k\right\rangle \leq \frac{H \log(SA)}{\eta} + \eta SAK + \frac{\eta H \log(H/\zeta)}{\gamma}.
  \end{align*}
  This finishes the proof.
\end{proof}

\section{Supporting Lemmas}
\label{appendix:supporting-lemmas}
In this section, we present some supporting lemmas, which are useful in our proofs.

First, we introduce the following two lemmas, which are used in the analysis of super martingale.
\begin{myLemma}[Lemma 7 of~\citet{ICML'20:Faury-improved-logistic}]
  \label{lem:variance-bound}
  Let $\epsilon$ be a centered random variable of variance $\sigma^2$ such that $|\epsilon|\leq 1$ almost surely. Then for all $\lambda \in [-1, 1]$ we have $\E[\exp(\lambda \epsilon)] \leq 1 + \lambda^2 \sigma^2$.
\end{myLemma}

\begin{myLemma}[Theorem 1 of~\citet{AISTATS'11:Alina-contextual-bandits}]
  \label{lem:bernstein-martingale}
  Let $Y_1, \ldots, Y_K$ be a martingale difference sequence with respect to a filtration $\F_1, \ldots, \F_K$. Suppose that $|Y_k| \leq R$ for all $k \in [K]$. Then, for any $\zeta \in (0, 1)$ and $\lambda \in [0, 1/R]$, with probability at least $1 - \zeta$, we have $\sum_{k=1}^K Y_k \leq \lambda \sum_{k=1}^K \E[Y_k^2 | \F_{k-1}] + \frac{\log(1/\zeta)}{\lambda}$.
\end{myLemma}

Then, we introduce the lemma that guarantees the biased loss estimator is close to the true loss function.

\begin{myLemma}[Lemma 11 of~\citet{ICML'20:bandit-unknown-chijin}]
  \label{lem:biased-loss-estimator}
  For any sequence of functions $\alpha_1, \ldots, \alpha_K$ such that $\alpha_k \in [0, 2\gamma]^{S \times A}$ if $\F_{k-1, H}$-measurable for all $k \in [k]$, with probability at least $1 - \zeta$, we have
  \begin{align*}
    \sum_{k=1}^K \sum_{(s, a) \in \mathcal{S} \times \mathcal{A}} \alpha_k(s, a)\left(\widehat{\ell}_k(s, a)-\frac{q_k(s, a)}{u_k(s, a)} \ell_k(s, a)\right) \leq H \log \left(\frac{H}{\zeta}\right).
  \end{align*}
\end{myLemma}

Next, we present the self-normalized concentration and determinant-trace lemma of~\citet{NIPS'11:AY-linear-bandits}.

\begin{myLemma}[Theorem 1 of~\citet{NIPS'11:AY-linear-bandits}]
  \label{lem:self-normalized}
  Let $\left\{\F_t\right\}_{t=0}^{\infty}$ be a filtration. Let $\left\{\eta_t\right\}_{t=1}^{\infty}$ be a real-valued stochastic process such that $\eta_t$ is $\F_t$-measurable and $\eta_t$ is conditionally zero-mean $R$-sub-Gaussian for $R \geq 0$ i.e. $\forall \lambda \in \mathbb{R}, \E\left[e^{\lambda \eta_t} \mid \F_{t-1}\right] \leq \exp \left({\lambda^2 R^2}/{2}\right)$. Let $\left\{X_t\right\}_{t=1}^{\infty}$ be an $\mathbb{R}^d$-valued stochastic process such that $X_t$ is $\F_{t-1}$-measurable. Assume that $V$ is a $d \times d$ positive definite matrix. For any $t \geq 0$, define $\bar{V}_t=V+\sum_{s=1}^t X_s X_s^{\top}$ and $S_t=\sum_{s=1}^t \eta_s X_s$. Then, for any $\zeta>0$, with probability at least $1-\zeta$, for all $t \geq 0$,
  \begin{align*}
    \left\|S_t\right\|_{\bar{V}_t^{-1}}^2 \leq 2 R^2 \log \left(\frac{\det\left(\bar{V}_t\right)^{1 / 2} \det (V)^{-1 / 2}}{\zeta}\right).
  \end{align*}
\end{myLemma}

\begin{myLemma}[Lemma 10 of~\citet{NIPS'11:AY-linear-bandits}]
  \label{lem:det-trace-inequality}
  Suppose $x_1, \ldots, x_t \in \R^d$ and for any $1 \leq s \leq t$, $\norm{x_s}_2 \leq L$. Let $V_t = \lambda I_d + \sum_{s=1}^t x_s x_s^\top $ for some $\lambda \geq 0$. Then, for any $1 \leq s \leq t$, we have $\det(V_t) \leq (\lambda + {tL^2}/{d})^d$.
\end{myLemma}

Finally, we introduce the generalized elliptical potential lemma, which is designed specifically for our analysis.
\begin{myLemma}[Generalized elliptical potential lemma]
  \label{lem:generalized-elliptical-potential}
  Suppose $\x_1, \ldots, \x_t \in \R^{N\times d}$ and for any $1 \leq s \leq t$, $\norm{x_{s,i}}_2 \leq L$. Let $\Lambda_t = \lambda_t I_d + \sum_{s=1}^{t-1}\sum_{i=1}^N \x_{s,i} \x_{s,i}^\top$ with $\lambda_t \geq \lambda_{t-1}$ and $\lambda_1= 1$. Then, for any $1 \leq s \leq t$, we have
  \begin{align*}
    \sum_{s=1}^t \left(1 \wedge \sum_{i=1}^N \norm{\x_{s,i}}_{\Lambda_{s}^{-1}}\right) \leq 2d \log \left(\lambda_{t+1} + \frac{tNL^2}{d}\right).
  \end{align*}
\end{myLemma}
\begin{proof}
  By the definition of $\Lambda_t$, we have
  \begin{align*}
    \det(\Lambda_{t+1}) = \det \left(\Lambda_t  + \sum_{i=1}^N \x_{t,i} \x_{t,i}^\top + (\lambda_{t+1} - \lambda_t)I_d\right)            \geq \det \left(\Lambda_t  + \sum_{i=1}^N \x_{t,i} \x_{t,i}^\top \right) = \det(\Lambda_t) \left(1 + \sum_{i=1}^N \norm{\x_{t,i}}_{\Lambda_t^{-1}}^2\right),
  \end{align*}
  where the inequality holds by the fact that $\lambda_{t+1} \geq \lambda_t$. Taking log from both sides and summing from $s=1$ to $t$:
  \begin{align*}
    \sum_{s=1}^t \log \left(1 + \sum_{i=1}^N \norm{\x_{s,i}}_{\Lambda_s^{-1}}^2\right) = \log \left(\frac{\det(\Lambda_{t+1})}{\det(\Lambda_1)}\right) \leq d \log(\lambda_{t+1} + \frac{tNL^2}{d})
  \end{align*}
  where the last inequality holds by the determinant-trace inequality in Lemma~\ref{lem:det-trace-inequality}. For any $a$ such that $0 \leq a \leq 1$, it holds that $a \leq 2 \log(1 + a)$. Thus, we have
  \begin{align*}
    \sum_{s=1}^t \bigg(1 \wedge \sum_{i=1}^N \norm{\x_{s,i}}_{\Lambda_s^{-1}}\bigg) \leq 2 \sum_{s=1}^t \log \bigg(1 + \sum_{i=1}^N \norm{\x_{s,i}}_{\Lambda_s^{-1}}^2\bigg) \leq 2d \log \bigg(\lambda_{t+1} + \frac{tNL^2}{d} \bigg).
  \end{align*}
  This completes the proof.
\end{proof}

\end{document}